\documentclass{article}
     \PassOptionsToPackage{numbers, compress}{natbib}

     \usepackage[preprint]{neurips_2020}

\usepackage[utf8]{inputenc} 
\usepackage[T1]{fontenc}    
\usepackage{hyperref}
\usepackage{url}            
\usepackage{booktabs}       
\usepackage{amsfonts}       
\usepackage{nicefrac}       
\usepackage{microtype}      
\usepackage{graphicx,graphics}
\usepackage{subfigure}
\usepackage{smile}
\usepackage{bm}
\usepackage[algo2e, ruled, vlined]{algorithm2e}
\usepackage{psfrag,amsmath,amsthm,amssymb}
\usepackage{comment}
\usepackage{wrapfig}
\usepackage{xcolor}
\usepackage{enumitem}
\usepackage{yfonts}
\usepackage{eufrak}

\newcommand{\Real}{\mathbb{R}}

\newcommand{\Tra}{^{\sf T}} 

\newcommand{\V}[1]{{\bm{\mathbf{\MakeLowercase{#1}}}}} 
\newcommand{\M}[1]{{\bm{\mathbf{\MakeUppercase{#1}}}}} 

\theoremstyle{plain}
\newtheorem{thm}{Theorem}
\newtheorem{lem}{Lemma}

\theoremstyle{definition}
\newtheorem{defn}{Definition}

\newtheorem{case}{Case}

\theoremstyle{remark}

\title{Online Forgetting Process for Linear Regression Models}

\author{%
  Yuantong Li \\
  Department of Statistics\\
  Purdue University\\
  \texttt{li3551@purdue.edu} \\
  \And
  Chi-Hua Wang \\
  Department of Statistics\\
  Purdue University\\
  \texttt{wang3667@purdue.edu} \\
  \AND
  Guang Cheng\\
  Department of Statistics\\
  Purdue University\\
  \texttt{chengg@purdue.edu} \\
}

\begin{document}

\maketitle

\begin{abstract}
Motivated by the EU's "Right To Be Forgotten" regulation, we initiate a study of statistical data deletion problems where users' data are accessible only for a limited period of time. This setting is formulated as an online supervised learning task with \textit{constant memory limit}. We propose a deletion-aware algorithm \texttt{FIFD-OLS} for the low dimensional case, and witness a catastrophic rank swinging phenomenon due to the data deletion operation, which leads to statistical inefficiency. As a remedy, we propose the \texttt{FIFD-Adaptive Ridge} algorithm with a novel online regularization scheme, that effectively offsets the uncertainty from deletion. In theory, we provide the cumulative regret upper bound for both online forgetting algorithms. In the experiment, we showed \texttt{FIFD-Adaptive Ridge} outperforms the ridge regression algorithm with fixed regularization level, and hopefully sheds some light on more complex statistical models.    
\end{abstract}

\section{Introduction}

Today many internet companies and organizations are facing the situation that \textit{certain} individual's data can no longer be used to train their models by legal requirements. Such circumstance forces companies and organizations to \textit{delete} their database on the demand of users' willing to be forgotten. On the ground of the `Right to be Forgotten', regularization is established in many countries and states' laws, including the EU's General Data Protection Regulation (GDPR)\citep{GDPR2016}, and the recent California Consumer Privacy Act Right (CCPA) \citep{CCPA2018}, which also stipulates users to require companies and organizations such as Google, Facebook, Twitter, etc to forget and delete these personal data to protect their privacy.
Besides, users also have the right to request the platform to delete his/her obsolete data at any time or only to authorize the platform to hold his/her personal information such as photos, emails, etc, only for a \textit{limited} period. Unfortunately, given these data are typically incrementally collected online, it is a disaster for the machine learning model to forget these data in chronological order. Such a challenge opens the needs to design and analyze \textit{deletion-aware} online machine learning method.

\begin{figure*}[t]
  \centering
  \includegraphics[scale=.39]{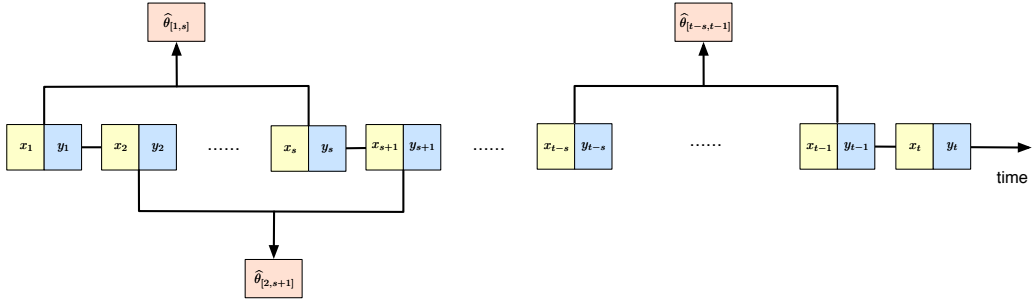}
  \caption{Online Forgetting Process with constant memory limit $s$.}
  \label{fig:fig-0}
\end{figure*}




In this paper, we propose and investigate a class of online learning procedure, termed \textit{online forgetting process}, to adapt users' requests to delete their data before a specific time bar. To proceed with the discussion, we consider a special deletion practice, termed \textit{First In First Delete} (FIFD), to address the scenario that the users only authorize their data for a limited period. (See Figure \ref{fig:fig-0} for an illustration of the online forgetting process with constant memory limit $s$.)  In FIFD,
the agent is required to \textit{delete} the oldest data as soon as receiving the latest data, to meet a  \textit{constant memory limit}. 
The FIFD deletion practice is inspired by the situation that, 
the system may only use data from the past three months to train their machine learning model to offer service for new customers \citep{JP2019,GOOGL2020}. 
The proposed online forgetting process is an \textit{online} extension of recent works that consider offline data deletion \citep{izzo2020approximate,ginart2019making, bourtoule2019machine} 
or detect data been forgotten or not \citep{liu2020have}. 

In such a `machine forgetting' setting, we aim to determine its difference with standard statistical machine learning methods via an online regression framework. To accommodate such limited authority, we provide solutions on designing \textit{deletion-aware} online linear regression algorithms, and discuss the harm due to the "constant memory limit" setting. Such a setting is challenging for a general online statistical learning task since the "sample size" never grows to infinity but stays a constant size along the whole learning process. As an evil consequence, statistical efficiency is never improving as the time step grows due to the constant memory limit.

\textbf{Our Contribution.} We first investigate the online forgetting process in the ordinary linear regression. We find a new phenomenon defined as \textit{Rank Swinging Phenomenon}, which exists in the online forgetting process.
If the deleted data can be fully represented by the data memory, then it will not introduce any regret. Otherwise, it will introduce extra regret to make this online forgetting process task not online learnable. The rank swinging phenomenon plays such a role that it indirectly represents the dissimilarity between the deletion data and the new data to affect the instantaneous regret. Besides, if the gram matrix does not have full rank, it will cause the \textit{adaptive constant} $\zeta$ to be unstable and then the confidence ellipsoid will become wider. Taking both of these effects into consideration, the order of the \texttt{FIFD-OLS}'s regret will become linear in time horizon $T$.

The rank swinging phenomenon affects the regret scale and destabilizes the \texttt{FIFD-OLS} algorithm. 
Thus, to remedy this problem, we propose the \texttt{FIFD-Adaptive Ridge} to offset this phenomenon because when we add the regularization parameter, the gram matrix will have full rank. Different from using the fixed regularization parameter in the standard online learning model, we use the martingale technique to adaptively select the regularization parameter over time. 

\textbf{Notation.} Throughout this paper, we denote $[T]$ as the set $\{1, 2, \ldots, T\}$. $|S|$ denotes the number of elements for any collection $S$. We use $\norm{x}_{p}$ to denote the $p$-norm of a vector $x \in \Real^{d}$ and $\norm{x}_{\infty} = \text{sup}_{i}{|x_{i}|}$. For any vector $v \in \Real^{d}$, notation $\mathcal{P}(v) \equiv \{i| v_{i} > 0\}$ denotes the indexes of positive coordinate of $v$ and $\mathcal{N}(v) \equiv \{i| v_{i} < 0\}$ denotes the indexes of negative coordinate of $v$. $\mathcal{P}_{\text{min}}(v) \equiv \min \{v_{i}| i \in \mathcal{P}(v)\}$ denotes the minimum value of $v_{i}$, where $i$ is in the indexes of positive coordinate and $\mathcal{N}_{\text{max}}(v) \equiv \max \{v_{i}| i \in \mathcal{N}(v)\}$ denotes the maximum value of $v_{i}$, where $i$ is in the indexes of negative coordinate.

For a positive semi-definite matrix $\Phi \in \Real^{d\times d}_{\succeq 0}$, $\lambda_{\text{min}}(\Phi)$ denotes the minimum eigenvalue of $\Phi$. We denote $\Phi^{-}$ as the generalized inverse of $\Phi$ if it satisfies the condition $\Phi = \Phi \Phi^{-} \Phi$. The weighted 2-norm of vector $x \in \Real^{d}$ with respect to positive definite matrix $\Phi$ is defined by $\norm{x}_{\Phi} = \sqrt{x\Tra \Phi x}$ . The inner product is denoted by $\langle \cdot, \cdot \rangle$ and the weighted inner-product is denoted by $x\Tra \Phi y = \langle x, y \rangle_{\Phi}$. 
For any sequence $\{x_{t}\}_{t=0}^{\infty}$, we denote matrix $\V{x}_{[a,b]} = \{x_{a}, x_{a+1}, \ldots, x_{b}\}$ and $\norm{\V{x}_{[a,b]}}_{\infty} = \text{sup}_{i,j} |\V{x}_{[a,b]}|_{(i,j)}$. For any matrix $\V{x}_{[a,b]}$, notation $\Phi_{[a,b]} = \sum_{t=a}^{b}x_{t}x_{t}^{\top}$ represents a gram matrix with constant memory limit $s = b - a + 1$ and notation $\Phi_{\lambda, [a,b]} = \sum_{t=a}^{b}x_{t}x_{t}^{\top} + \lambda \M{I}_{d\times d}$ represents a gram matrix 
with ridge hyperparameter $\lambda$.

\section{Statistical Data Deletion Problem}

At each time step $t\in [T]$, where $T$ is a finite time horizon, the learner receives a context-response pair  $z_{t} = (x_t, y_{t})$, where $x_{t}\in \Real^{d}$  is a $d$-dimensional context and $y_{t} \in \Real$ is the response. The observed sequence of context $\{x_{t}\}_{t \geq 1}$ 
are drawn i.i.d from a distribution of $\mathcal{P}_{\mathcal{X}}$ with a bounded support $\mathcal{X} \subset \Real^{d}$ and  $\norm{x_{t}}_{2} \leq L$. 
Let $D_{[t-s: t-1]} =\{z_{i}\}_{i=t-s}^{t-1}$ denote the data collected at $[t-s, t-1]$ following the FIFD scheme.

We assume that for all $t\in [T]$, the response $y_{t}$ is a linear combination of context $x_{t}$; formally, 
\begin{equation} 
    y_{t} = \langle x_{t}, \theta_{\star} \rangle + \epsilon_{t}
    \label{linear regression},
\end{equation}
where $\theta_{\star}\in \Real^{d}$ is the \textit{target parameter} that summarizes the relation between the context $x_{t}$ and response $y_{t}$. 
The noise  $\epsilon_{t}$'s are drawn independently from $\sigma-$subgaussian distribution. That is, for every $\alpha \in \Real$, it is satisfied that $\mathbb{E}[\exp(\alpha \epsilon_{t})] \leq \exp(\alpha^2\sigma^2/2)$. 

Under the proposed FIFD scheme with a constant memory limit $s$, 
the algorithm $\mathcal{A}$ at time $t$ only can keep 
the information of historical data from time step $t-s$ to time step $t-1$ and then to make the prediction, and previous data before time step $t-s-1$ are not allowed to be kept and needs to be deleted or forgotten. The algorithm $\mathcal{A}$ is required to make total $T-s$ number of predictions in the time interval $[s+1,T]$.

To be more precise, the algorithm $\mathcal{A}$ first receives a context $x_{t}$ at time step $t$, and make a prediction based only on the information in previous $s$ time steps $D_{[t-s: t-1]}$ ($|D_{[t-s: t-1]}|=s$) and hence the agent forgets the information up to the time step $t-s-1$. The predicted value $\hat{y}_{t}$ computed by the algorithm $\mathcal{A}$ based on information $D_{[t-s: t-1]}$ and current the context $x_{t}$, which  is denoted by
\begin{equation}
    \hat{y}_{t} = \mathcal{A}(x_{t}, D_{[t-s: t-1]}).
\end{equation}
After prediction, the algorithm $\mathcal{A}$ receives the true response $y_{t}$ and suffers a \textit{pseudo regret}, $r(\hat{y}_{t}, y_{t})= (\langle x_{t}, \theta_{\star} \rangle - \hat{y}_{t})^2$. 
We define the cumulative (pseudo)-regret of the algorithm $\mathcal{A}$ 
up to time horizon $T$ as
\begin{equation}
    \label{eq:regret_def}
    R_{T, s}(\mathcal{A})
    \equiv 
    \sum_{t=s+1}^{T} (\langle x_{t}, \theta_{\star} \rangle - \hat{y}_{t})^2.
\end{equation}

Our theoretical goal is to explore the relationship between constant memory limit $s$ and the cumulative  regret $R_{T,s}(\mathcal{A})$. We proposed two algorithms \texttt{FIFD-OLS} and \texttt{FIFD-Adaptive Ridge} and studied their cumulative regret, respectively. 
In particular, we discusses the effect of constant memory limit $s$, dimension $d$, subguassian parameter $\sigma$, and confidence level $1-\delta$ on the obtained cumulative regret $R_{T, s}(\mathcal{A})$ for these two algorithms. For example, what's the effect of constant memory $s$ on the order of the regret $R_{T, s}(\mathcal{A})$ if other parameters keep constant. In other words, how many data do we need to keep in order to achieve the satisfied performance under the FIFD scheme. Besides, is there any amazing or unexpected phenomenon both in the experiment and theory occurred when the agent used the ordinary least square method under the FIFD scheme and how to improve it?

\textbf{Remark.}  The FIFD scheme can also be generalized to the standard online learning paradigm when we add two and delete one data or add more and delete one data. These settings will automatically transfer to the standard online learning model when $T$ becomes large. We presented the simulation result to illustrate it.

\section{FIFD OLS}

In this section, we present the FIFD - ordinary least square regression (\texttt{FIFD-OLS}) algorithm under \textit{``First In First Delete''} scheme and the corresponding confidence ellipsoid for the FIFD-OLS estimator and regret analysis. Besides, the rank swinging phenomenon will be discussed in section 3.4. 


\subsection{FIFD-OLS Algorithm}
The \texttt{FIFD-OLS} algorithm uses the \textit{least square estimator} based on the constant data memory from time window $[t-s,t-1]$, defined as
$\hat{\theta}_{[t-s,t-1]} = \Phi_{[t-s,t-1]}^{-1}\big[ \sum_{i=t-s}^{t-1}y_{i}x_{i}\big]$.
Then an incremental update formula for $\hat{\theta}$ from time window $[t-s,t-1]$ to $[t-s+1,t]$ is showed as follows.

\textbf{OLS incremental update}: \textit{At time step t, the estimator $\hat{\theta}_{[t-s+1,t]}$ is updated by previous estimator $\hat{\theta}_{[t-s,t-1]}$ and new data $z_{t}$},
\begin{equation}
\begin{aligned}
    \hat{\theta}_{[t-s+1,t]}
    = &f(\Phi_{[t-s,t-1]}^{-1}, x_{t-s}, x_{t})\cdot g(\hat{\theta}_{[t-s,t-1]}, \Phi_{[t-s,t-1]}, z_{t-s}, z_{t}),
\end{aligned}
\end{equation}
\textit{where $f(\Phi_{[t-s,t-1]}^{-1}, x_{t-s}, x_{t})$ is defined as}
\begin{equation}
	\begin{aligned}
    	\Gamma(\Phi_{[t-s,t-1]}^{-1})
    	- &(x_{t-s}^\top \Gamma(\Phi_{[t-s,t-1]}^{-1})-1)^{-1}\big[ 
    	\Gamma(\Phi_{[t-s,t-1]}^{-1})
    	x_{t-s}
    	x_{t-s}^\top 
    	\Gamma(\Phi_{[t-s,t-1]}^{-1})
    	\big] 
	\end{aligned}
\end{equation}
\textit{with $\Gamma(\Phi_{[t-s,t-1]}^{-1})\equiv
\Phi_{[t-s,t-1]}^{-1}-(x_{t}^\top \Phi_{[t-s,t-1]}^{-1} x_{t}+1)^{-1}
    \big[ 
    \Phi_{[t-s,t-1]}^{-1}
    x_{t}
    x_{t}^\top 
    \Phi_{[t-s,t-1]}^{-1}
    \big]$ and $g(\hat{\theta}_{[t-s,t-1]}, \Phi_{[t-s,t-1]}, z_{t-s}, z_{t})$ is defined as $\Phi_{[t-s,t-1]}\hat{\theta}_{[t-s,t-1]} +y_{t}x_{t} -y_{t-s}x_{t-s}$. 
    The detailed online incremental update is in Appendix \ref{Appendix Online Update}.
}

The algorithm has two steps. First, $f$-step is to update the inverse of gram matrix from $\Phi_{[t-s,t-1]}^{-1}$ to $\Phi_{[t-s+1,t]}^{-1}$ and the Woodbury formula \cite{woodbury1950inverting} is applied to reduce the time complexity to $\mathcal{O}(s^2d)$ compared with directly computing the inverse of gram matrix; Secondly, $g$-step is a simple algebra and uses the definition of adding and forgetting data to update the least square estimator.
The detailed update procedure of \texttt{FIFD-OLS} is displayed in Algorithm \ref{algo:1}.

\begin{figure}[h]
\begin{algorithm2e}[H]
    \label{algo:1}
	\DontPrintSemicolon
	\SetAlgoLined
	\SetKwInOut{Input}{Input}\SetKwInOut{Output}{Output}
	Given parameters: $s$, $T$, $\delta$ \;
	\For{$t \in \{s+1, ..., T\}$}{
		Observe $x_{t}$ \\
		$\Phi_{[t-s,t-1]} = \V{x}_{[t-s,t-1]}\Tra \V{x}_{[t-s,t-1]}$ \\
		\eIf{t = s+1}{
			$\hat{\theta}_{[1,s]} = \Phi_{[1,s]}^{-1}\V{x}_{[1,s]}\Tra \V{y}_{[1,s]}$\;
			}{
			$\hat{\theta}_{[t-s,t-1]} = f(\Phi_{[t-s-1,t-2]}^{-1}, x_{t-s-1}, x_{t})\cdot g(\hat{\theta}_{[t-s-1,t-2]}, \Phi_{[t-s-1,t-2]},z_{t-s-1}, z_{t-1})$\;
		}
		Predict $\hat{y}_{t} = \langle \hat{\theta}_{[t-s,t-1]}, x_{t} \rangle$ and observe $y_{t}$ \\
		Compute loss $r_{t} = |\hat{y}_{t} - \langle \theta^{\star}, x_{t} \rangle|$ \\
		Delete data $z_{t-s} = (x_{t-s}, y_{t-s})$
	}
	\caption{\texttt{FIFD-OLS}}
\end{algorithm2e}
\end{figure}


\subsection{Confidence Ellipsoid for FIFD-OLS}
Our first contribution is to obtain a confidence ellipsoid for the OLS method based on the sample set collected under the FIFD scheme, showed in Lemma \ref{lm:OLS Confidence Ellipsoid}.

\begin{lem}
\label{lm:OLS Confidence Ellipsoid}(\textit{FIFD-OLS Confidence Ellipsoid}) \textit{For any $\delta > 0$, if the event $\lambda_{\text{min}}(\Phi_{[t-s,t-1]}/s) > \phi^2_{[t-s,t-1]} > 0$ holds, with probability at least $1-\delta$, for all $t\geq s+1$, $\theta_{\star}$ lies in the set}
	\begin{equation}\label{eq:OLS_ellip_4}
	    \begin{aligned}
	    	C_{[t-s,t-1]} = 
	    	\bigg\{
	    	&\theta \in \Real^{d}: 
	    	\norm{\hat{\theta}_{[t-s,t-1]} - \theta}_{\Phi_{[t-s,t-1]}} \leq 
	        \sigma
            q_{[t-s,t-1]}
            \sqrt{(2d/s)\log(2d/\delta)}
	    	\bigg\}
	    \end{aligned}    
	\end{equation}
where $q_{[t-s,t-1]} = \norm{\V{x}_{[t-s,t-1]}}_{\infty}/\phi^2_{ [t-s,t-1]}$ is the adaptive constant and we denote $\beta_{[t-s, t-1]}$ as the RHS bound.
\end{lem}
\textit{Proof.} The key step to get the confidence ellipsoid is to use the martingale noise technique presented in \cite{bastani2020online} to obtain an adaptive bound.  The detailed proof can be obtained in Appendix \ref{Appendix-FIFD-OLS Confidence Ellipsoid}.
Besides, the confidence ellipsoid in \eqref{eq:OLS_ellip_4} requires the information of minimum eigenvalue of gram matrix $\Phi_{[t-s,t-1]}$ and infinity norm of our observations with constant data memory $s$. Thus, it is an adaptive confidence ellipsoid following the change of data.

\subsection{FIFD OLS Regret}
The following theorem provides the regret upper bounds for the FIFD-OLS algorithm.
\begin{thm}
\label{thm:OLS Regret Bound}
(Regret Upper Bound of The FIFD-OLS Algorithm). Assume that for all $t \in [s+1, T-s]$ and ${X_{t}}$ is i.i.d random variables with distribution $\mathcal{P}_{\mathcal{X}}$. With probability at least $1-\delta \in [0,1]$, for all $T>s, s\geq d$, we have an upper bound on the cumulative regret at time $T$:
\begin{equation}
    \begin{aligned}
        &R_{T, s}(\mathcal{A}_{OLS})\leq 
        2\sigma \zeta \sqrt{(d/s) \log(2d/\delta)(T-s)
        \left( d \log (sL^{2}/d) +(T-s) \right)}, 
    \end{aligned}
\end{equation}
where the adaptive constant $\zeta = \underset{s+1 \leq t\leq T}{\mathrm{max}} q_{[t-s,t-1]}$.

\textit{Proof.}
We provide a roadmap for the proof of Theorem \ref{thm:OLS Regret Bound}. The proof is motivated by \citep{abbasi2011improved, bastani2020online}. We first prove Lemma \ref{lm:OLS Confidence Ellipsoid} to obtain a confidence ellipsoid holding for the FIFD-OLS estimator. Then we use the confidence ellipsoid, to sum up the regret over time in Lemma \ref{lem: LRT-lemma} and find a key term called \textit{FRT} (defined below equation \ref{eq:LRT}), which affects the derivation of regret upper bound a lot. The detailed proof of this theorem can be found in the appendix \ref{Appendix-FIFD-OLS Regret}.
\end{thm}

\textbf{Remarks.} We develop the \texttt{FIFD-OLS} algorithm and prove an upper bound of the cumulative regret in order of $\mathcal{O}(\sigma \zeta[(d/s)(\log{2d}/\delta)]^{\frac{1}{2}}T)$. The agent using this algorithm cannot improve the performance of this algorithm because it has a constant memory limit $s$ to update the estimated parameters. What's more, the adaptive constant $q_{[t-s, t-1]}$ is unstable caused by the forgetting process. 
The main factor causing the oscillation of gram matrix $\Phi_{[t-s, t-1]}$ is that the minimum eigenvalue is close to zero. In other words, the gram matrix is full rank at some time. Therefore, the adaptive constant $q_{[t-s, t-1]}$ will go to infinity at such time points. So the generalization's ability of \texttt{FIFD-OLS} is poor.

Following the derivation of the regret upper bound of \texttt{FIFD-OLS} in Theorem \ref{thm:OLS Regret Bound}, we find an interesting phenomenon, which we called \textit{Rank Swinging Phenomenon} (defined below Definition \ref{Rank Swinging Phenomenon}). This phenomenon will un-stabilize the \texttt{FIFD-OLS} algorithm and introduce some extreme bad events, which results in a larger value of the regret upper bound of the \texttt{FIFD-OLS} algorithm.

\subsection{Rank Swinging Phenomenon}
Below we give the definition of the \textit{Rank Swinging Phenomenon}.
\begin{defn}
\label{Rank Swinging Phenomenon}
(\textit{Rank Swinging Phenomenon})
\textit{At time $t$, when we delete data $x_{t-s}$ and add data $x_{t}$, the gram matrix switching from $\Phi_{[t-s, t-1]}$ to $\Phi_{[t-s+1, t]}$ will cause its rank increasing or decreasing by 1 or 0 if $\text{Rank}(\Phi_{[t-s, t-1]}) \leq d$,
\begin{equation}
	\text{Rank}(\Phi_{[t-s+1, t]}) = 
	\begin{cases}
		\text{Rank}(\Phi_{[t-s, t-1]}) + 1, \text{Case \ref{case:2}} \\
		\text{Rank}(\Phi_{[t-s, t-1]}), \text{Case \ref{case:3}}\\
		\text{Rank}(\Phi_{[t-s, t-1]}) - 1, \text{Case \ref{case:4}} 
	\end{cases}
\end{equation}
where four examples are illustrated in Appendix \ref{Appendix-Example-Rank-Swinging}. The real example can be found in Figure \ref{fig:rank_switching}.}
\end{defn}


The rank swinging phenomenon results in unstable regret, which is measured by the term called `Forgetting Regret Term'
 (FRT) caused by deletion  at time $t$. 
\begin{defn}
(\textit{Forgetting Regret Term})
\textit{
\begin{equation}
\label{eq:LRT}
\begin{aligned}
    &\text{FRT}_{[t-s, t-1]} =\norm{x_{t-s}}_{\Phi_{[t-s,t-1]}^{-1}}^{2}
    (\norm{x_{t}}_{\Phi_{[t-s,t-1]}^{-1}}^{2} 
    \sin^{2}(\theta, \Phi_{[t-s,t-1]}^{-1}) + 1),
\end{aligned}
\end{equation}
where $\text{FRT}_{[t-s, t-1]} \in [0,2]$. 
}
\end{defn}

Let's use $\sin^{2}(\theta, \Phi_{[t-s,t-1]}^{-1})$ denote the dissimilarity between $x_{t-s}$ and $x_{t}$ under $\Phi_{[t-s,t-1]}^{-1}$ as follows, 
\begin{equation*}
    \begin{aligned}
        \sin^{2}(\theta, \Phi_{[t-s,t-1]}^{-1}) &= \sin^{2}_{\Phi_{[t-s,t-1]}^{-1}}\theta  
        = 1 - \cos^{2}(\theta, \Phi_{[t-s,t-1]}^{-1})
    \end{aligned}
\end{equation*}
\begin{equation*}
    \begin{aligned}
        \cos(\theta, \Phi_{[t-s,t-1]}^{-1}) = \frac{\langle x_{t}, x_{t-s} \rangle_{\Phi_{[t-s,t-1]}^{-1}}}{ \norm{x_{t}}_{\Phi_{[t-s,t-1]}^{-1}} \norm{x_{t-s}}_{\Phi_{[t-s,t-1]}^{-1}}}
    \end{aligned}
\end{equation*}

\begin{figure}[t]
  \centering
    \includegraphics[scale=0.48]{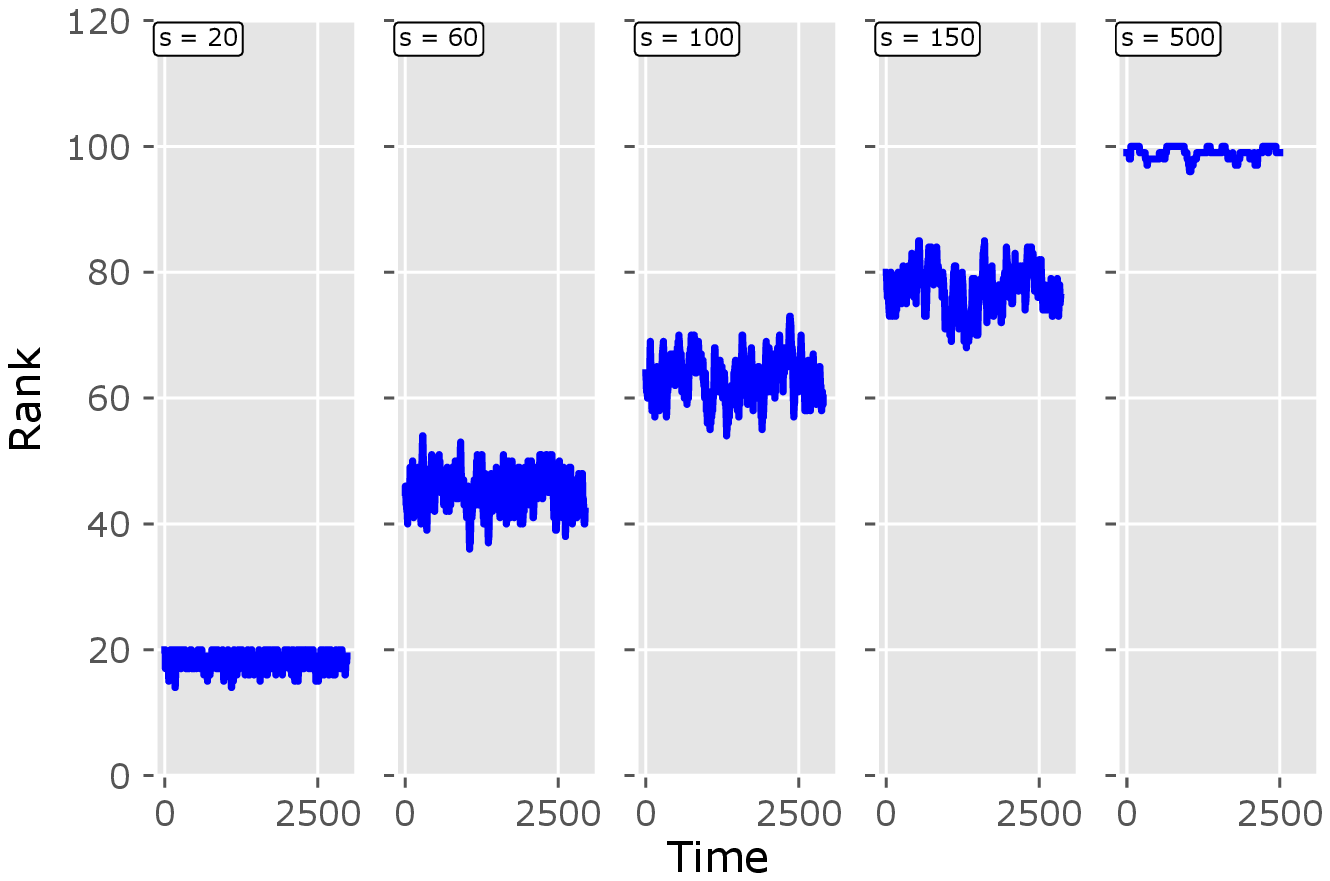}
    \caption{Rank switching phenomenon with memory limit $s = 20, 60, 100, 150, 500$ and $d = 100$ from left to right. $X_t \sim \text{Unif}(e_1, e_2, ... e_{100})$. The rank of gram matrix can decrease due to deletion operation.}
    \label{fig:rank_switching}
\end{figure}

\textbf{Remark.} If $\text{FRT}_{[t-s, t-1]} = 0$, then it won't introduce extra regret to avoid it being online learnable which means the deleting data can be fully represented by the rest of data and the deletion operation won't sabotage the representation power of the algorithm. If $\text{FRT}_{[t-s, t-1]} \neq 0$, it will introduce extra regret during this online forgetting process, which means that the deletion data can't be fully represented by the rest of the data and the deletion operation will sabotage the representation power of the algorithm.

$\text{FRT}$ is determined by two terms the \textit{`deleted term'} $\norm{x_{t-s}}_{\Phi_{[t-s,t-1]}^{-1}}^{2}$ and the \textit{`dissimilarity term'} $\norm{x_{t}}_{\Phi_{[t-s,t-1]}^{-1}}^{2}
\sin^{2}(\theta, \Phi_{[t-s,t-1]}^{-1})$. If this deleted term $\norm{x_{t-s}}_{\Phi_{[t-s,t-1]}^{-1}}^{2}$ is zero, then $\text{FRT}$ won't introduce any extra regret, no matter how large the weight term 
$\norm{x_{t}}_{\Phi_{[t-s,t-1]}^{-1}}^{2}\sin^{2}(\theta, \Phi_{[t-s,t-1]}^{-1})$ is. The larger the dissimilarity term, the more regret it will introduce if the deleted term is not zero since the new data introduces new direction in the representation space. If $\text{FRT}_{[t-s, t-1]} = 0, \forall s < t \leq T$ , we will achieve order $\mathcal{O}(\sqrt{T})$ cumulative regret, which is online learnable as expected. However, this hardly happens. Therefore, the cumulative regret under the FIFD scheme is usually $\mathcal{O}(T)$, which is proved in Theorem  \ref{thm:OLS Regret Bound} and Theorem \ref{thm:Adaptive Ridge Regret Bound}. 

In the following Lemma \ref{lem: LRT-lemma}, we show the importance of FRT in getting the upper bound of Theorem \ref{thm:OLS Regret Bound} and Theorem \ref{thm:Adaptive Ridge Regret Bound}. Here for the sake of simplicity, if $\text{Rank}(\Phi_{[t-s,t-1]}) < d$,  we will not discriminate the notation of generalized inverse $\Phi_{[t-s,t-1]}^{-}$ and inverse $\Phi_{[t-s,t-1]}^{-1}$, and uniformly use $\Phi_{[t-s,t-1]}^{-1}$ to represent the inverse of $\Phi_{[t-s,t-1]}$. 
\begin{lem}
\label{lem: LRT-lemma}
The cumulative regret of FIFD-OLS is partially determined by FRT at each time step, and the cumulative representation term is 
\begin{equation}
    \begin{aligned}
    \sum_{t=s+1}^{T} \norm{x_{t}}_{\Phi_{[t-s,t-1]}^{-1}}^{2} 
    \leq
        2\eta_{\text{OLS}}
        + 
        \sum_{t =s+1}^{T} 
        \text{FRT}_{[t-s, t-1]}
    \end{aligned}
\end{equation}
where $\eta_{\text{OLS}} = \log(\text{det}(\Phi_{[T-s, T]}))$ is a constant based on data time window $[T-s, T]$.
\end{lem}

\textit{Proof.} Let's first denote $r_{t}$ as the instantaneous regret at time $t$ and decompose the instantaneous regret,
$r_{t} = \langle \hat{\theta}_{[t-s,t-1]}, x_{t} \rangle - \langle \theta_{\star}, x_{t} \rangle 
        \leq \sqrt{\beta_{[t-s,t-1]}(\delta)}  \norm{x_{t}}_{\Phi_{[t-s,t-1]}^{-1}}$,
where the inequality is from Lemma \ref{lm:OLS Confidence Ellipsoid}. 
Thus, with probability at least $1-\delta$, for all $T>s$,
\begin{equation*}
    \begin{aligned}
        R_{T, s}(\mathcal{A}_{OLS})
        &\leq \sqrt{(T-s) \sum_{t = s+1}^{T} r_{t}^2} \\
        &\leq 
        \sqrt{(T-s) \underset{s+1 \leq t\leq T}{
            \mathrm{max}} \beta_{[t-s,t-1]}(\delta) \sum_{t = s+1}^{T} \norm{x_{t}}_{\Phi_{[t-s,t-1]}^{-1}}^{2}}.
    \end{aligned}
\end{equation*}



\section{FIFD-Adaptive Ridge}
To remedy the rank switching phenomenon, we take advantage of the ridge regression method to avoid this happening in the online forgetting process.
In this section, we present the FIFD - adaptive ridge regression (\texttt{FIFD-Adaptive Ridge}) algorithm under \textit{``First In First Delete''} scheme and the corresponding confidence ellipsoid for the FIFD-Adaptive Ridge estimator and regret analysis.

\subsection{FIFD-Adaptive Ridge Algorithm}

The \texttt{FIFD-Adaptive Ridge} algorithm use the \textit{ridge estimator}  $\hat{\theta}_{\lambda, [t-s,t-1]}$ to predict the response. The definition of it for time window $[t-s, t-1]$ is showed as follows. 

\textbf{Adaptive Ridge update}: $$\hat{\theta}_{\lambda, [t-s,t-1]} = \Phi_{\lambda, [t-s,t-1]}^{-1}\big[ \sum_{i=t-s}^{t-1}y_{i}x_{i}\big].$$


Then we display the adaptive choice of hyperparameter $\lambda$ for time window $[t-s, t-1]$.
\begin{lem}
\label{lm:ridge parameter chioce}
(Adaptive Ridge Parameter $\lambda_{[t-s,t-1]}$) 
If the event $\lambda_{\text{min}}(\Phi_{\lambda, [t-s,t-1]}/s) > \phi^{2}_{\lambda, [t-s,t-1]}$ holds, with probability $1-\delta$, for any $$\chi(\delta) > \sigma \norm{\V{x}_{[a,b]}}_{\infty} \sqrt{(2d/s)\log(2d/\delta)}
/\phi^2_{\lambda, [a,b]},$$ 
we have a control of $L_{2}$ estimation error that $\text{Pr}\left[ \norm{\hat{\theta}_{\lambda,[t-s,t-1]} - \theta_{\star}}_{2} \leq \chi(\delta) \right] \geq 1-\delta$, and to satisfy this condition, we select the adaptive $\lambda_{[t-s,t-1]}$ as follows for the limited time window $[t-s,t-1]$,
\begin{equation}
\label{eq: Ridge hyperparameter selection}
    \begin{aligned}
        \lambda_{[t-s,t-1]} \leq 
        \sigma \norm{\V{x}_{[t-1,t-s]}}_{\infty} \sqrt{2s \log(2d/\delta)}  /\norm{\theta_{\star}}_{\infty}.
    \end{aligned}
\end{equation}
The detailed derivation can be found in Appendix \ref{Appendix FIFD-Adaptive Ridge Regret}.
\end{lem}
The detailed update procedure of \texttt{FIFD-Adaptive Ridge} is displayed in Algorithm \ref{algo:2}.

\begin{minipage}{1\textwidth}
\begin{algorithm2e}[H]
    \label{algo:2}
	\DontPrintSemicolon
	\SetAlgoLined
	\SetKwInOut{Input}{Input}\SetKwInOut{Output}{Output}
	Given parameters: $s$, $T$, $\delta$ \;
	\For{$t \in \{s+1, ..., T\}$}{
		Observes $x_{t}$ \\
		Calculates $\norm{\V{x}_{[t-1,t-s]}}_{\infty} = \underset{1\leq i \leq s, 1\leq j \leq d}{\max} \V{x}_{(i,j)}$\;
		Gets estimated noise $\hat{\sigma}_{[t-s,t-1]} = \text{SD}(\V{y}_{[t-s,t-1]})$\;
		Sets the adaptive penalized parameter $\lambda_{[t-s,t-1]} = \sqrt{2s} \hat{\sigma}_{[t-s,t-1]} \norm{\V{x}_{[t-s,t-1]}}_{\infty} \sqrt{\log 2d/\delta}$\;
		$\Phi_{\lambda,[t-s,t-1]} = \V{x}_{[t-s,t-1]}\Tra \V{x}_{[t-s,t-1]} + \lambda_{[t-s,t-1]} \M{I}$\;
		$\hat{\theta}_{[t-s,t-1]} = \Phi_{\lambda,[t-s,t-1]}^{-1} \V{x}_{[t-s,t-1]}\Tra \V{y}_{[t-s,t-1]}$\;
		Predicts $\hat{y}_{t} = \langle \hat{\theta}_{[t-s,t-1]}, x_{t} \rangle$ and observes $y_{t}$ \\
		Computes loss $r_{t} = |\hat{y}_{t} - \langle \theta^{\star}, x_{t} \rangle|$ \\
		Deletes data $z_{t-s} = (x_{t-s}, y_{t-s})$
	}
	\caption{\texttt{FIFD-Adaptive Ridge}}
\end{algorithm2e}
\end{minipage}


\subsection{Confidence Ellipsoid for FIFD-Adaptive Ridge}

Before we moving to the confidence ellipsoid of the adaptive ridge estimator, we define some notions about the true parameter $\theta_{\star}$, where $\mathcal{P}_{\text{min}}(\theta_{\star})$ represents the weakest positive signal and $\mathcal{N}_{\text{max}}(\theta_{\star})$ serves as the weakest negative signal. To make the adaptive ridge confidence ellipsoid more simplified, without loss of generality, we assume the following assumption.

\textbf{Assumption 1}. 
\textit{(Weakest Positive to Strongest Signal Ratio) We assume positive coordinate of $\theta_{\star}$ dominates the bad events happening,}
\begin{equation}
    \label{WPS ratio}
    \begin{aligned}
        &\text{WPSSR} = \frac{\mathcal{P}_{\text{min}}(\theta_{\star})}
        {\norm{\theta_{\star}}_{\infty}}\leq \frac{-\sqrt{C_3} + \sqrt{C_3+ s^{2} \log{\frac{2d}{\delta}}\log{\frac{12|\mathcal{P}(\theta_{\star})|}{\delta}}}}{s\log{\frac{2d}{\delta}}},
    \end{aligned}
\end{equation}
where $C_3 = \log{\frac{6d}{\delta}}\log{\frac{2d}{\delta}}.$ The WPSSR is monotone increasing in $s$ and $\mathcal{P}(\theta_{\star})$, and is monotone decreasing in $d$. In most cases, the LHS is greater than one, such as $s=100, d=110, \delta=0.05, |\mathcal{P}(\theta_{\star})| = 30$, then $\text{WPS Ratio}$ needs to be less than 1.02, which is satisfied this assumption.

If Assumption 1 holds, a high probability confidence ellipsoid can be obtained for FIFD-Adaptive ridge.
\begin{lem}
\label{lm:Ridge Confidence Ellipsoid}
(\textit{FIFD-Adaptive Ridge Confidence Ellipsoid}) \textit{For any $\delta \in [0, 1]$, with probability at least $1-\delta$, for all $t\geq s+1$, with Assumption 1 and the event $\lambda_{\text{min}}(\Phi_{\lambda, [t-s,t-1]}/s) > \phi^2_{\lambda, [t-s,t-1]} > 0$ holds, 
then $\theta_{\star}$ lies in the set}
	\begin{equation}
	    \begin{aligned}
	    	C_{\lambda, [t-s,t-1]} = \bigg\{
	    	&\theta \in \Real^{d}: 
	    	\norm{\hat{\theta}_{\lambda, [t-s,t-1]} - \theta}_{\Phi_{\lambda, [t-s,t-1]}}\leq 
	        \sigma \kappa \nu
            q_{\lambda, [t-s,t-1]} 
            \sqrt{d/2s}
	    	\bigg\},
	    \end{aligned}    
	\end{equation}
where $q_{\lambda, [t-s,t-1]} = \norm{\V{x}_{[t-1,t-s]}}_{\infty}/ \phi^2_{\lambda, [t-s,t-1]}$, $\kappa = \sqrt{\log^{2}(6|\mathcal{P}(\theta_{\star})|/ \delta)/\log(2d/ \delta)}$, and $\nu =\norm{\theta_{\star}}_{\infty}/ \mathcal{P}_{\text{min}}(\theta_{\star})$.
\end{lem}
\textit{Proof.} Key steps. The same technique used in Lemma \ref{lm:OLS Confidence Ellipsoid} is applied here to obtain an adaptive confidence ellipsoid for ridge estimator. Here we assume Assumption 1 to make the confidence ellipsoid (\ref{lm:Ridge Confidence Ellipsoid}) more simplified.  
The detailed proof can be obtained in the Appendix \ref{Appendix-FIFD-Adaptive Ridge Confidence Ellipsoid}.
 
\textbf{Remark.} The order of constant memory limit $s$ in the confidence ellipsoid for FIFD-OLS and FIFD-Adaptive ridge's confidence ellipsoids are  $\mathcal{O}(\sqrt{s})$. The adaptive ridge confidence ellipsoid requires the information of minimum eigenvalue of gram matrix $\Phi_{\lambda, [t-s,t-1]}$ with hyperparameter $\lambda$ and infinity norm of data $\norm{\V{x}_{[t-1,t-s]}}_{\infty}$, which is similar as FIFD-OLS confidence ellipsoid required. The benefits of introduction of $\lambda$ avoids the singularity of gram matrix $\Phi$, which will stabilize the algorithm FIFD-Adaptive ridge and make the confidence ellipsoid narrow in most times. Besides, it is an adaptive confidence ellipsoid.

\subsection{FIFD Adaptive Ridge Regret}

In the following theorem, we provide the regret upper bound for the \texttt{FIFD-Adaptive Ridge} method.

\begin{thm}\label{thm:Adaptive Ridge Regret Bound}
(Regret Upper Bound of the FIFD-Adaptive Ridge algorithm) \textit{With Assumption 1 and with probability at least $1-\delta$, the cumulative regret satisfies:
\begin{equation}
    \begin{aligned}
        R_{T, s}(\mathcal{A}_{Ridge}) \leq 
            \sigma \kappa \nu \zeta_{\lambda}  \sqrt{(d/s) (T-s)
            [\eta_{\text{Ridge}} + (T-s) ] 
            }
    \end{aligned}
\end{equation}
where $\zeta_{\lambda} = \underset{s+1 \leq t\leq T}{\mathrm{max}} \frac{\norm{\V{x}_{[t-1,t-s]}}_{\infty}}{\phi^2_{\lambda, [t-s,t-1]}}$ is the maximum adaptive constant, $\eta_{\text{Ridge}} = d \log (sL^{2}/d + \lambda_{[T-s, T-1]}) - \log{C_{2}(\V{\phi})}$ is a constant related to the last data memory, $C_{2}(\V{\phi}) = \prod_{t=s+1}^{T}(1 + \frac{s}{\phi^2_{\lambda, [t-s+1,t]} - \lambda_{\Delta, [t-s+1,t]}} \lambda_{\Delta, [t-s+1,t]})$ is a constant close to 1, and $\lambda_{\Delta, [t-s+1,t]} = \lambda_{[t-s+1,t]} - \lambda_{[t-s,t-1]}$ represents the change of $\lambda$ over time steps.}
\end{thm}

\textit{Proof.} The key step is to use the same technical lemma used in the OLS setting to show a confidence ellipsoid holds for the adaptive ridge estimator. Lemma \ref{lem: LRT-lemma} is used to compute FRT and then sums up FTRs to get the cumulative regret upper bound. The detailed proof of this theorem can be found in Appendix \ref{Appendix FIFD-Adaptive Ridge Regret}.

\textbf{Remarks.} \textit{(Relationship between regret upper bound and parameters)} We develop the FIFD-Adaptive Ridge algorithm and then provide a regret upper bound in order
$\mathcal{O}(\sigma\kappa\zeta_{\lambda}L[d/s]^{\frac{1}{2}}T)$. This order represents the relationship between
the regret, noise $\sigma$, dimension $d$, constant memory limit $s$, signal level $\nu$, and confidence level $1-\delta$. Since the agent always keeps the constant memory limit $s$ data to update the estimated parameters, thus the agent can't improve the performance of the algorithm. Therefore, the regret is $\mathcal{O}(T)$ with respect to the decision number. Besides, we find that using the ridge estimator can offset the rank swinging phenomenon since Ridge's gram matrix is always full rank.

\section{Simulation Experiments}

We compare the \texttt{FIFD-Adaptive Ridge} method where $\lambda$ is chosen under Lemma \ref{lm:ridge parameter chioce} with the baselines with pre-tuned $\lambda$. For all experiments unless otherwise specified, we choose $T = 3000, \delta = 0.05, d = 100, L =1$ for all simulation settings. All results are averaged over 100 runs.

\textbf{Simulation settings.} The setting for constant memory limit $s$ is $20, 40, 60, 80$ and subgaussian parameter $\sigma$ is $1, 2, 3$. Context $\{x_{t}\}_{t\in [T]}$ is generated from $\text{N}(\V{0}_{d}, \M{I}_{d\times d})$, and then we normalize it. 
The response $\{y_{t}\}_{t\in [T]}$ is generated by $y_{t} = \langle x, \theta_{\star} \rangle + \epsilon_{t}$, where $\theta_{\star} \sim \text{N}(\V{0}_{d}, \M{I}_{d\times d})$ and then normalize it with $\norm{\theta_{\star}}_{2} = 1$. We set $\epsilon_{t} \overset{\text{i.i.d}}{\sim}\sigma$-subguassian. The adaptive ridge's noise $\sigma$ in the simulation is estimated by $\hat{\sigma}_{[t-s, t-1]} = \text{sd}(\V{y}_{[t-s, t-1]}), t\in [s+1, T]$.

\textbf{Hyperparameter settings for competing methods.} The hyperparameter $\lambda$ we select for $\sigma = 1$ setting is $\{1,10,100\}$. Since by the relationship of hyperparameter $\lambda$ and noise level $\sigma$, we know they are in linear order. Thus, for the $\sigma = 2$ setting, we set $\lambda = \{2, 20, 200\}$ and for the $\sigma = 3$ setting, we set $\lambda = \{3, 30, 300\}$.  The adaptive ridge hyperparameter is automatically calculated according to Lemma \ref{lm:ridge parameter chioce} and we assume $\norm{\theta_{\star}}_{\infty} = 1$ since $\norm{\theta_{\star}}_{2} = 1$.

\begin{figure}[t!]
\centering
    \includegraphics[scale=0.57]{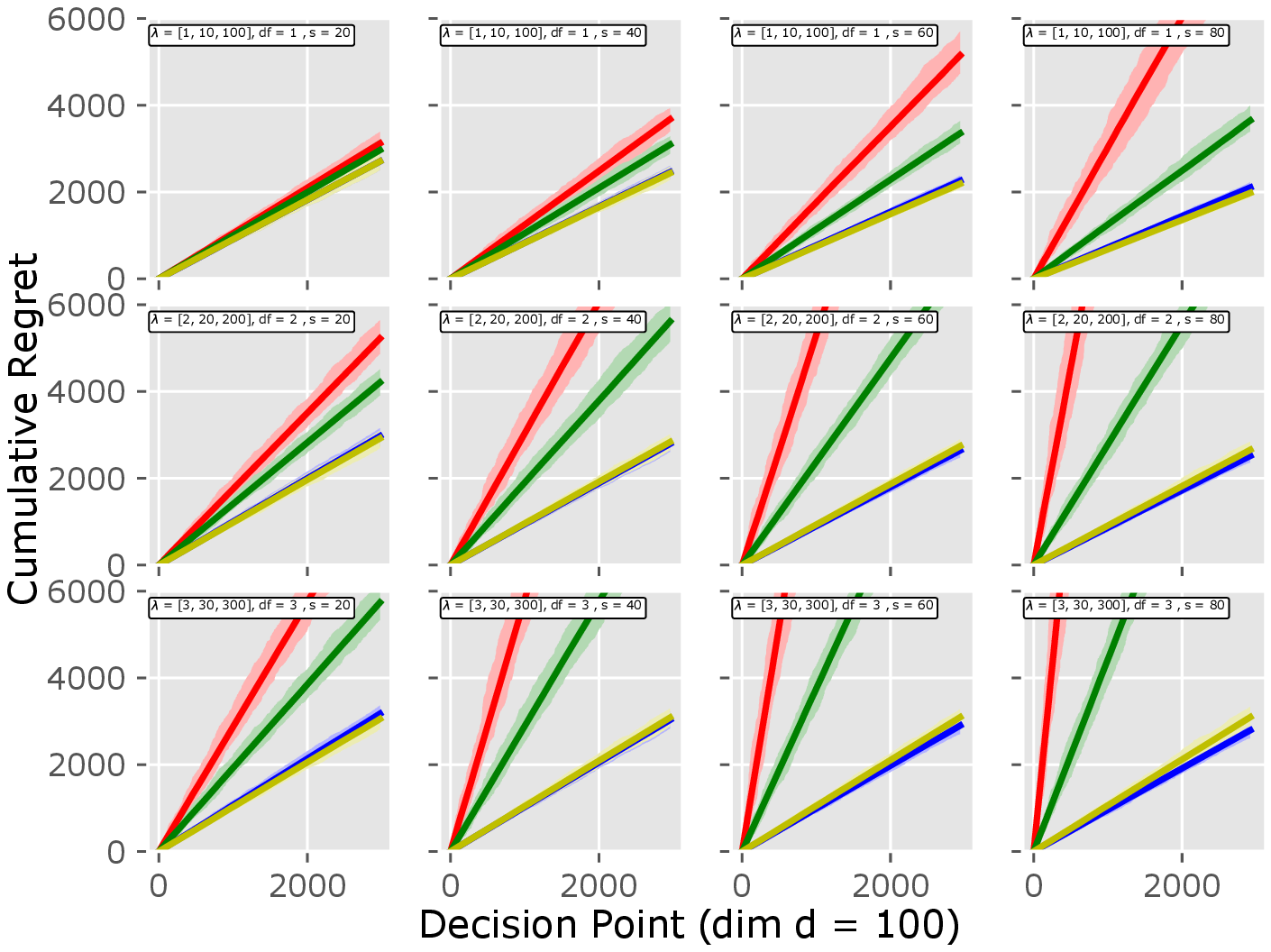}\vspace{-0.7em}
    \caption{Comparison of cumulative regret between of the Adaptive Ridge method and Fixed Ridge method. The error bars represent the standard error of the mean regret over 100 runs. The blue line represents the Adaptive Ridge. The green line represents the Fixed Ridge method with $\lambda = \{1,2,3\}$ for each row. The red line represents the Fixed-Ridge method with $\lambda = \{10,20,30\}$ for each row. The red line represents the Fixed-Ridge method with $\lambda = \{100,200,300\}$ for each row.}
\label{fig:Ridge Regret}
\end{figure} 
\textbf{Results.}
In figure \ref{fig:Ridge Regret}, we present the simulation result of relationship between the cumulative regret $\Real_{T,s}(\mathcal{A}_{Ridge})$, noise level $\sigma$, hyperparameter $\lambda$, and constant memory limit $s$. We find all of the cumulative regret are linear in time horizon $T$ just with different constant levels.

From three rows, as the memory limit $s$ increases, we find that the cumulative regret of the adaptive ridge (blue line) decreases. 
Since we know that as the sample size increases, the regret will decrease in theory. For each column, when we fix the memory limit $s$, as the noise level $\sigma$ increases, the regret of FIFD-Adaptive Ridge and Fixed Ridge increase with different hyperparameters.

Overall, we find that the FIFD-Adaptive Ridge method is more robust to the change of noise level $\sigma$ when we view figure \ref{fig:Ridge Regret} by row. Fixed Ridge methods with different hyperparameters such as green line and red line are sensitive to the change of noise $\sigma$ and memory limit $s$ because Fixed Ridge doesn't have any prior to adaptive the data. Although it performs well in the top left, it doesn't work well in other settings. The yellow line has relative comparable performance compared with the Adaptive Ridge method since its hyperparameter $\lambda$ is determined by the knowledge from Adaptive Ridge.
The FIFD-Adaptive Ridge is always the best (2nd row and 3rd row) or close to the best (1st row) choice of $\lambda$ among all of these settings, which means that the Adaptive Ridge method is robust to the large noise.  

Besides, the Adaptive Ridge method can save computational cost compared with the fixed Ridge method,  which needs cross-validation to select the optimal $\lambda$. 

In addition, more results about the choice of hyperparameter $\lambda$ and $L_{2}$ estimation error of the estimator can be found in Appendix \ref{Appendix Additional Simulation Results}.


\begin{figure}
    \centering
    \includegraphics[scale=0.57]{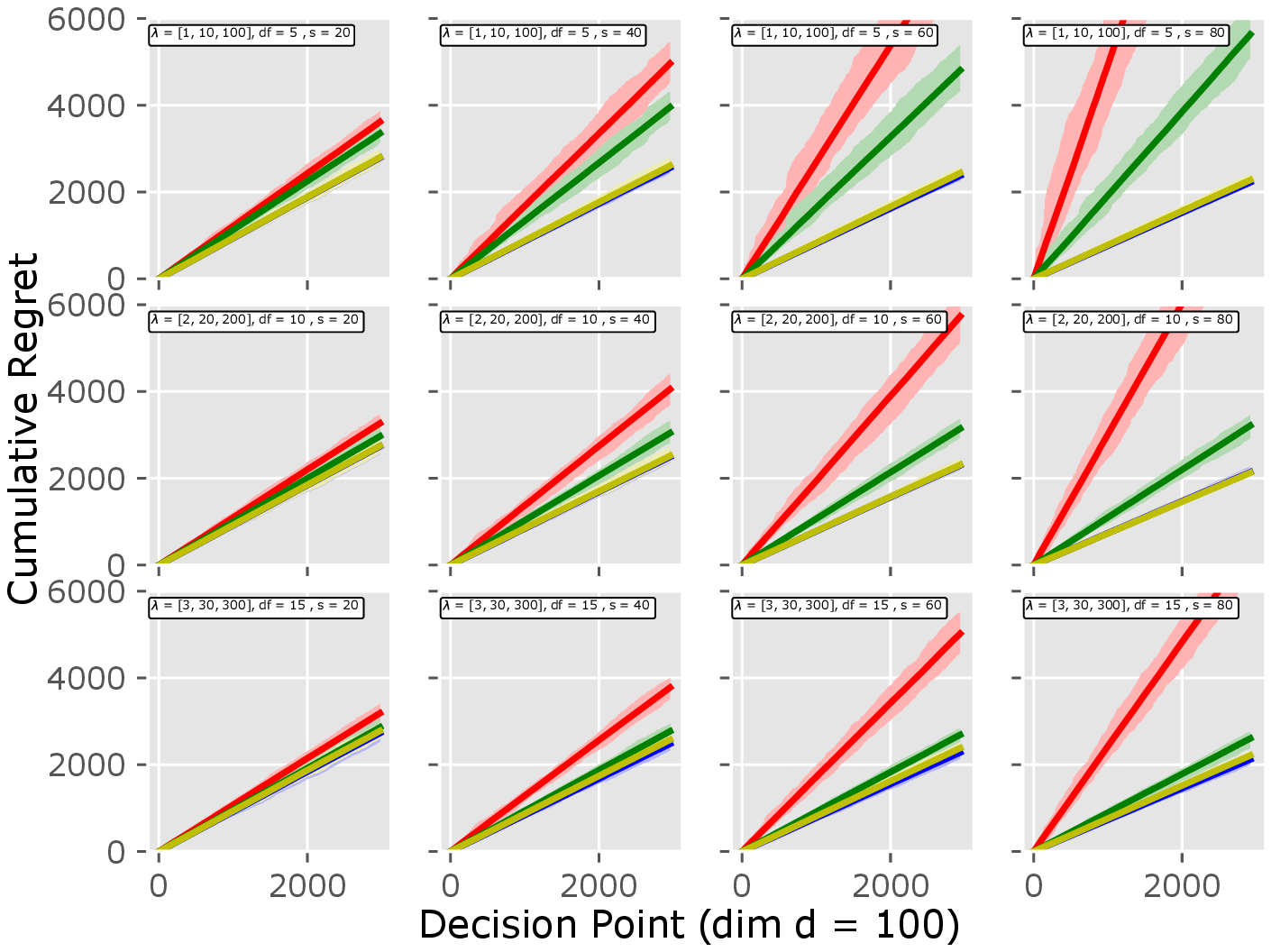}
    \caption{Here $\epsilon_t \sim t_{df}$ with degree of freedom $\text{df} = \{5,10,15\}$. This setting investigate the robustness of the proposed online regularization scheme.
    These lines' colors are the same as they represent in figure \ref{fig:Ridge Regret}.}
    \label{fig:T Regret}
\end{figure}

\textbf{Heavy Tail Noise Distribution.} Moreover, we also test the robustness of these methods concerning different noise. Here we assume $\epsilon_{t} \sim t_{df}, \forall t\in [T]$.
The degree of freedom \textit{df} is set to be 5, 10, 15. As we know when \textit{df} is greater than 30, it behaves like the normal distribution. When \textit{df} is small, $t$-distribution has much heavier tails. Thus, in figure \ref{fig:T Regret}, we find that when \textit{df} increases, the cumulative regret is decreasing since the error is more like the normal distribution, which can be well captured by our algorithm. However, fixed Ridge methods are not robust to the change of noise distribution, especially green line and red line.

\begin{figure}
    \centering
    \includegraphics[scale=0.57]{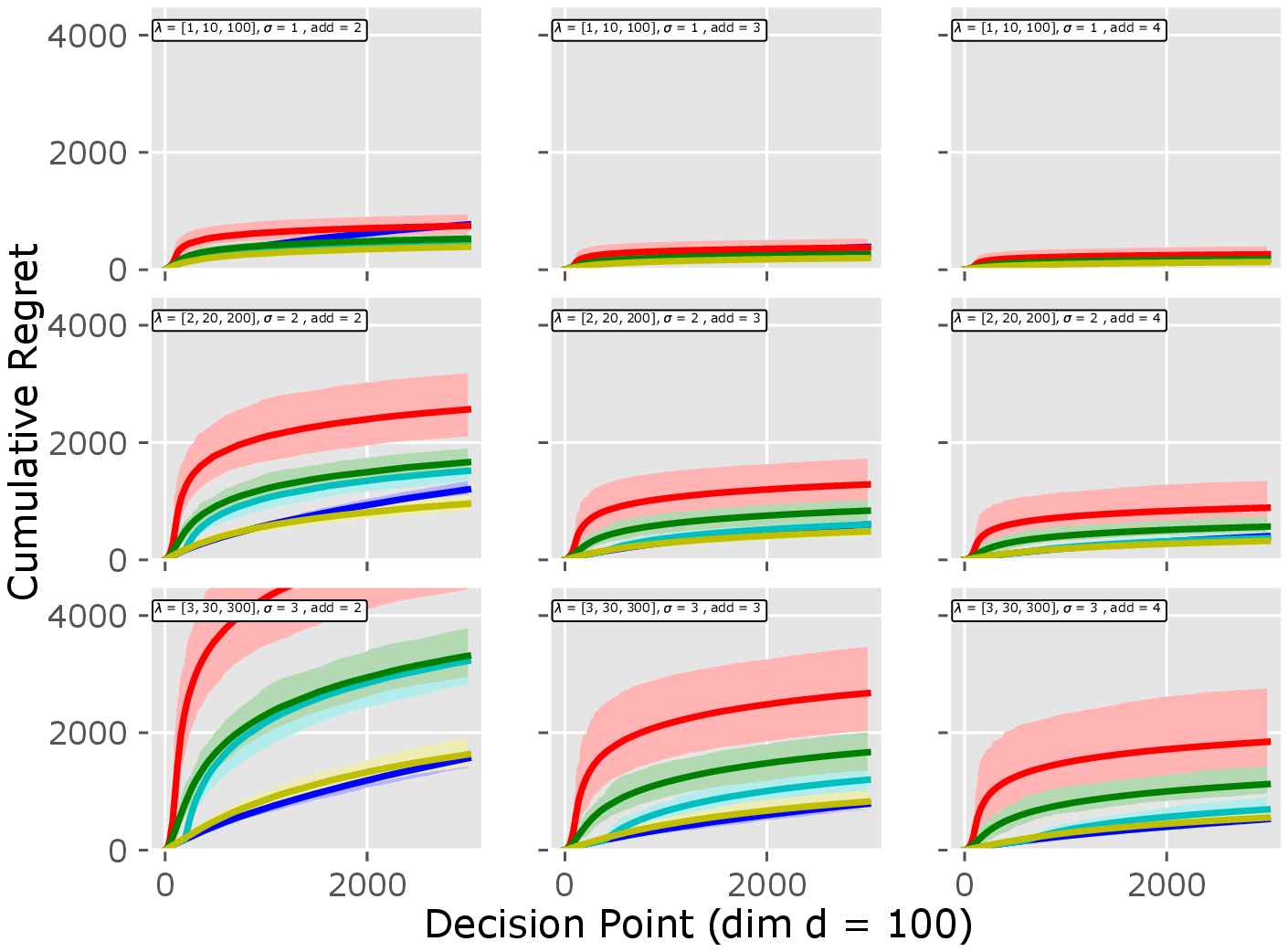}
    \caption{Switching from Ridge to OLS $\&$ ($+k, -1$): Comparison of cumulative regret between of the Switching Adaptive Ridge method when $s \geq 2d$ and Fixed Ridge method. The error bars represent the standard error of the mean regret over 100 runs.
    The cyan line represents the Switching Adaptive Ridge. Other lines' colors are the same as they presented in Figure \ref{fig:Ridge Regret}. 
    }
    \label{fig:AdaptRidgeSwitch Regret}
\end{figure}
\textbf{Results for $(+k, -1)$ addition-deletion operation.}
Here we test the pattern when we add more than one data point at each time step, such as $k = 2,3,4$, and delete one data, denoted as $(+k, -1)$ addition-deletion pattern.
In figure \ref{fig:AdaptRidgeSwitch Regret}, the 1st, 2nd, 3rd column are cases that delete one data after receiving 2, 3, 4 data, respectively. Each row has different noise level $\sigma = 1,2,3$. From the figure, we notice that the cumulative regret in each subplot has an increasing sublinear pattern. For the first row, the noise level is the smallest, which shows the sublinear pattern quickly. For other rows, all subplots show increasing sublinear patterns sooner or later, which satisfied our anticipation that when adding $k > 1$ data points and delete one data point at each time step, it finally will convert the normal online learn regime. However, the noise level will affect the time it moves to the sublinear pattern.
 
\textbf{Switching from Ridge to OLS.} Besides, we also consider when the agent accumulates many data such as $n > 2d$, where $n$ represents the number of data the agent has at some time point t, and then transfer the algorithm from the Adaptive ridge method to the OLS method. This method called \textit{Switching Adaptive Ridge} represented as the cyan line in figure \ref{fig:AdaptRidgeSwitch Regret}. We see that when the noise level is relatively small, it performs well. However, when $\sigma$ is large, it works worse than the green line (fixed Ridge with prior knowledge about $\lambda$) and the blue line (Adaptive Ridge).

\section{Discussion and Conclusion}

In this paper, we have proposed two online forgetting algorithms under the \texttt{FIFD} scheme and provide the theoretical regret upper bound.
To our knowledge, this is the first, theoretically well-proved work in the online forgetting process under the FIFD scheme.

Besides, we find the existence of rank swinging phenomenon in the online least square regression and tackle it using the online adaptive ridge regression. In the future, we hope we can provide the lower bound for these two algorithms and design other online forgetting algorithms under the FIFD scheme.

\bibliographystyle{plainnat}
\nocite*{}
\bibliography{FIFD}

\clearpage
\appendix

\onecolumn

\begin{center}
    \Large Supplement to ``Online Forgetting Process for Linear Regression Models''
\end{center}

This appendix contains the proof of the theoretical results, rank swinging phenomenon examples,  and simulation results to illustrate the FIFD-Adaptive Ridge algorithm's performance. For the sake of simplicity, in the following proof of lemmas and theorems in appendix,  we denote $a = t-s, b = t-1$, and then the \textit{constant memory limit} $s$ equals to $b-a+1$ for the \textit{limited time window} $[t-s,t-1] = [a,b]$.

\section{FIFD-OLS Confidence Ellipsoid}
\label{Appendix-FIFD-OLS Confidence Ellipsoid}

\begin{lemma}(Bernstein Concentration).
    \label{lm:bernstein concentration}
    \textit{Let $\{D_{k}, \textfrak{S}_{k}\}_{k=1}^{\infty}$ be a martingale difference, and suppose that $D_{k}$ is a $\sigma$-subgaussian in an adapted sense, i.e., for all $\alpha \in \mathbb{R}$. $\mathbb{E}[e^{\alpha D_{k}}| \textfrak{S}_{k-1}] \leq e^{\frac{\alpha^2 \sigma^2}{2}}$ almost surely. Then, for all $t\geq 0$, $\text{Pr}[|\sum_{k=1}^{n} D_{k}| \geq t] \leq 2e^{-\frac{t^2}{2n\sigma^2}}$}.
\end{lemma}
Lemma \ref{lm:bernstein concentration} is from Theorem 2.3 of Wainwright (2019) \citep{wainwright2019high} when $\alpha_{\ast} = \alpha_{k} = 0$ and $\nu_{k} = \sigma$ for all $k$.

\begin{lemma} 
\textit{Define the event}
    \begin{equation*}
    \mathcal{F}(\lambda_{0}(\gamma)) \equiv \{ \underset{r \in [d]}{\text{max}}(2|\epsilon\Tra X^{(r)}|/n)\leq \lambda_{0}(\gamma)\}
    \end{equation*}
    \textit{where $X^{(r)}$ is the $r^{\text{th}}$ column of matrix $\M{X}$ and $\lambda_{0}(\gamma) \equiv 2\sigma x_{\text{max}} \sqrt{(\gamma^2 + 2 \log d)/n}$. Then, we have $\text{Pr}[\mathcal{F}(\lambda_{0}(\gamma))] \geq 1 - 2\exp[-\gamma^2/2]$.}
\end{lemma}
\begin{proof}
    Let $\textfrak{S}_{t}$ be the sigma algebra generated by random variables $X_{1}, \ldots, X_{t-1}$ and $Y_{1}, \ldots, Y_{t-1}$. First, using a union bound, we can write
    \begin{equation*}
    \text{Pr}[\mathcal{F}(\lambda_{0}(\gamma))] \geq 1 - \sum_{r = 1}^{d} \text{Pr}[|\epsilon \Tra X^{(r)}| > n \lambda_{0}(\gamma)/2]
    \end{equation*}
    Now, for each $r \in [d]$, let $D_{t, r} = \epsilon_{t}X_{t,r}$ and note that $D_{1,r}, \ldots, D_{n, r}$ is a martingale difference sequence adapted to the filtration $\textfrak{S}_{1} \subset \ldots \subset \textfrak{S}_{n}$ since $\mathbb{E}[\epsilon_{t}X_{t, r}| \textfrak{S}_{t}] = 0$. On the other hand, each $D_{t, r}$ is a $(x_{\text{max}}, \sigma)$-subgaussian random variable adapted to $\{\textfrak{S}_{t}\}_{t=1}^{n}$, since 
    \begin{equation*}
    \mathbb{E}(e^{\alpha D_{t, r}}| \textfrak{S}_{t-1}) \leq \mathbb{E}_{X_{t}}[e^{\alpha^{2}X_{t,r}^{2}\sigma^{2}/2}|\textfrak{S}_{t-1}] \leq e^{\alpha^{2}(x_{\text{max}}\sigma)^2}. 
    \end{equation*}
Then, using Lemma \ref{lm:bernstein concentration}, $\text{Pr}[\mathcal{F}(\lambda_{0}(\gamma))] \geq 1 - 2d \exp[-(\gamma^2 + 2\log{d})/2] =  1-2\exp[-\gamma^2/2]$.
\end{proof}

\begin{lemma}
\label{lm: Appendix OLS CE}
(\textit{FIFD-OLS Confidence Ellipsoid}) \textit{For any $\delta > 0$, if the event $\lambda_{\text{min}}(\Phi_{[t-s,t-1]}/s) > \phi^2_{[t-s,t-1]} > 0$ holds, with probability at least $1-\delta$, for all $t\geq s+1$, $\theta_{\star}$ lies in the set}
    \begin{equation}\label{eq:OLS_ellip}
        \begin{aligned}
            C_{[t-s,t-1]} = \bigg\{\theta \in \Real^{d}: 
            \norm{\hat{\theta}_{[t-s,t-1]} - \theta}_{\Phi_{[t-s,t-1]}}
            &\leq 
            \sigma
            q_{[t-s,t-1]}
            \sqrt{(2d/s)\log(2d/\delta)}
            \bigg\}
        \end{aligned}    
    \end{equation}
where $q_{[t-s,t-1]} = \norm{\V{x}_{[t-s,t-1]}}_{\infty}/\phi^2_{ [t-s,t-1]}$ is the adaptive constant and we denote $\beta_{[t-s, t-1]}$ as the RHS bound.
\end{lemma}
\begin{proof}
Notation $\hat{\Sigma}(\V{x}_{[a,b]})$ represents the normalized covariance matrix, so $\hat{\Sigma}(\V{x}_{[a,b]}) = \Phi_{[a,b]}/s = \V{x}_{[a,b]}\Tra \V{x}_{[a,b]}/s$. Note that, if the event $\lambda_{\min}(\hat{\Sigma}(\V{x}_{[a,b]})) > \phi^2_{[a,b]}$ holds,
\begin{equation}
    \begin{aligned}
        \norm{\hat{\theta}_{[a,b]} - \theta_{\star}}_{2} 
        &=\norm{(\V{x}_{[a,b]}\Tra \V{x}_{[a,b]})^{-1}\V{x}_{[a,b]}\Tra(\V{x}_{[a,b]} \theta_{\star} + \epsilon) - \theta_{\star} }_{2} 
        \\
        &=\norm{(\V{x}_{[a,b]}\Tra \V{x}_{[a,b]})^{-1}\V{x}_{[a,b]}\Tra \epsilon 
            + \theta_{\star} - \theta_{\star}}_{2}
        \\
        &=\norm{(\V{x}_{[a,b]}\Tra \V{x}_{[a,b]})^{-1}\V{x}_{[a,b]}\Tra \epsilon}_{2}
        \\
        &\leq \frac{1}{s\phi^2_{[a,b]}} \norm{\V{x}_{[a,b]}\Tra \epsilon}_2
    \end{aligned}
\end{equation}

Then, for any $\chi > 0$, we can write
\begin{equation}
    \begin{aligned}
        \text{Pr}\left[ \norm{\hat{\theta}_{[a,b]} - \theta_{\star}}_{2} \leq \chi \right] &\geq
        \text{Pr} \left[(\norm{\V{x}_{[a,b]}\Tra \epsilon}_{2} \leq s \chi \phi^2_{[a,b]}) \cap (\lambda_{\min}(\hat{\Sigma}_{[a,b]}) > \phi^2_{[a,b]})\right] \\
        &\geq 
        1 - \sum_{r=1}^{d} \text{Pr}\left[ |\epsilon \Tra \V{x}^{(r)}_{[a,b]}| > \frac{s\chi \phi^2_{[a,b]}}{\sqrt{d}}\right] - \text{Pr}\left[ \lambda_{\min}(\hat{\Sigma}_{[a,b]}) \leq \phi^2_{[a,b]}) \right] 
    \end{aligned}
\end{equation}
where we have let $\V{x}^{(r)}_{[a,b]}$ denote the $r^{\text{th}}$ column of $\V{x}_{[a,b]}$. We can expand $\epsilon\Tra \V{x}^{(r)}_{[a,b]} = \sum_{j \in [a ,b]} \epsilon(j) \V{x}_{j}^{(r)}$, where we note that $D_{j, r} \equiv \epsilon(k) \V{x}_{j}^{(r)}$ is a $x_{\max} \sigma$-subgaussian random variable, where $x_{\max} = \norm{\V{x}_{[a,b]}}_{\infty}$, conditioned on the sigma algebra $\textfrak{S}_{j-1}$ that is generated by random variable $X_1,..., X_{j-1}, y_{1},..., y_{j-1}$. Defining $D_{0, r} = 0$, the sequence $D_{0, r}, D_{1, r},..., D_{(b,r)}$ is a martingale difference sequence adapted to the filtration $\textfrak{S}_{1} \subset \textfrak{S}_{2} \subset ... \textfrak{S}_{b}$ since $E[\epsilon(j)X_{j}^{(r)}|\textfrak{S}_{j-1}] = 0$. Using Lemma \ref{lm:bernstein concentration},
\begin{equation}
    \begin{aligned}
    \label{eq:ridge case control prob}
        \text{Pr}&\left[ \norm{\hat{\theta}_{[a,b]} - \theta_{\star}}_{2} \leq \chi \right]
        \\
        &\geq 
        1 - \sum_{r=1}^{d} \text{Pr}\left[ |\epsilon \Tra \V{x}^{(r)}_{[a,b]}| > \frac{s
        \chi \phi^2_{[a,b]} 
        }{\sqrt{d}}\right]  - \text{Pr}\left[ \lambda_{\min}(\hat{\Sigma}_{[a,b]}) \leq \phi^2_{[a,b]}) \right] \\
        &\geq 1 - 2d\exp{\left[-\frac{s\chi^{2} \phi^{4}_{[a,b]}}{2d\norm{\V{x}_{[a,b]}}_{\infty}^{2}\sigma^2} \right]} -
        \text{Pr}\left[ \lambda_{\min}(\hat{\Sigma}_{[a,b]}) \leq \phi^2_{[a,b]}) \right],
    \end{aligned}    
\end{equation}
Since the event $\lambda_{\text{min}}(\Phi_{[t-s,t-1]}/s) > \phi^2_{[t-s,t-1]} > 0$ holds by the requirement of condition, then $\text{Pr}\left[ \lambda_{\min}(\hat{\Sigma}_{[a,b]}) \leq \phi^2_{[a,b]}) \right] = 0$. With probability $1-\delta$, we have
\begin{equation}
    \begin{aligned}
        & 1 - 2d\exp{\left[-\frac{s\chi^{2} \phi^{4}_{[a,b]}}{2d \norm{\V{x}_{[a,b]}}_{\infty}^{2} \sigma^2} \right]} \geq 1-\delta.
    \end{aligned}
\end{equation}
Hence we have
\begin{equation}
\begin{aligned}
    \chi(\delta, s, q_{[a,b], \sigma, d}) \geq \sigma q_{[a,b]} \sqrt{\frac{2d}{s}\log(\frac{2d}{\delta})}.
\end{aligned}
\end{equation}

where $q_{[a,b]} =  \norm{\V{x}_{[a,b]}}_{\infty}/\phi^2_{ [a,b]}$ is the adaptive constant for the limited time window $[a,b]$ . Besides, we denote $\beta_{[a, b]}$ as the RHS constant. Then with probability $1-\delta$, for the limited time window $[t-s, t-1]$ and all $t\geq s+1$, we have the get the \texttt{FIFD-OLS} confidence ellipsoid, $\theta_{\star}$ lies in the set
    \begin{equation}
        \begin{aligned}
            C_{[t-s,t-1]} = \bigg\{\theta \in \Real^{d}: 
            \norm{\hat{\theta}_{[t-s,t-1]} - \theta}_{\Phi_{[t-s,t-1]}}
            &\leq 
            \sigma
            q_{[t-s,t-1]}
            \sqrt{(2d/s)\log(2d/\delta)}
            \bigg\}.
        \end{aligned}    
    \end{equation}
\end{proof}

\section{FIFD-Adaptive Ridge Confidence Ellipsoid}
\label{Appendix-FIFD-Adaptive Ridge Confidence Ellipsoid}

\begin{lemma}
\label{lm:Appendix Ridge Confidence Ellipsoid}
(\textit{FIFD-Adaptive Ridge Confidence Ellipsoid}) \textit{For any $\delta \in [0, 1]$, with probability at least $1-\delta$, for all $t\geq s+1$, if condition equation \eqref{WPS ratio} is satisfied and the event $\lambda_{\text{min}}(\Phi_{\lambda, [t-s,t-1]}/s) > \phi^2_{\lambda, [t-s,t-1]} > 0$ holds, 
then $\theta_{\star}$ lies in the set}
    \begin{equation}
        \begin{aligned}
            C_{\lambda, [t-s,t-1]} = \bigg\{\theta \in \Real^{d}: 
            \norm{\hat{\theta}_{\lambda, [t-s,t-1]} - \theta}_{\Phi_{\lambda, [t-s,t-1]}}
            &\leq 
            \sigma \kappa \nu
            q_{\lambda, [t-s,t-1]} 
            \sqrt{d/2s}
            \bigg\},
        \end{aligned}    
    \end{equation}
where $q_{\lambda, [t-s,t-1]} = \norm{\V{x}_{[t-1,t-s]}}_{\infty}/ \phi^2_{\lambda, [t-s,t-1]}$, and $\kappa = \sqrt{\log^{2}(6|\mathcal{P}(\theta_{\star})|/ \delta)/\log(2d/ \delta)}$. $\nu =\norm{\theta_{\star}}_{\infty}/ \mathcal{P}_{\text{min}}(\theta_{\star})$ represents the strongest signal to weakest signal ratio.

\end{lemma}
\begin{proof}
For the sake of simplicity,
We first denote $\hat{\Sigma}_{\lambda}(\V{x}_{[a,b]}) = \Phi_{\lambda, [a,b]}/s$, $\hat{\Sigma}_{\lambda, [a,b]} = \hat{\Sigma}_{\lambda}(\V{x}_{[a,b]})$, $\lambda_{[a,b]} = \lambda$ for the limited time window $[t-s, t-1]$.
Note that, if the event $\lambda_{\min}(\hat{\Sigma}_{\lambda}(\V{x}_{[a,b]})) = \lambda_{\min}(\hat{\Sigma}_{\lambda, [a,b]}) > \phi^2_{\lambda, [a,b]} > 0$ holds, 
\begin{equation}
    \begin{aligned}
        \norm{&\hat{\theta}_{\lambda,[a,b]} - \theta_{\star}}_{2} \\
        &=\norm{(\V{x}_{[a,b]}\Tra \V{x}_{[a,b]} + \lambda \M{I} )^{-1}\V{x}_{[a,b]}\Tra(\V{x}_{[a,b]} \theta_{\star} + \epsilon) - \theta_{\star} }_{2} 
        \\
        &=\norm{(\V{x}_{[a,b]}\Tra \V{x}_{[a,b]} + 
         \lambda \M{I})^{-1}\V{x}_{[a,b]}\Tra \epsilon 
            + 
            (\V{x}_{[a,b]}\Tra \V{x}_{[a,b]} + \lambda \M{I})^{-1}(\V{x}_{[a,b]}\Tra \V{x}_{[a,b]} + \lambda \M{I})\theta_{\star} \\
        & \hspace{2cm}- \lambda (\V{x}_{[a,b]}\Tra \V{x}_{[a,b]} + \lambda \M{I})^{-1}\theta_{\star} - \theta_{\star}}_{2}
        \\
        &= \norm{(\V{x}_{[a,b]}\Tra \V{x}_{[a,b]} + 
         \lambda \M{I})^{-1}\V{x}\Tra \epsilon + \theta^{\star} 
         - \lambda (\V{x}_{[a,b]}\Tra \V{x}_{[a,b]} + \lambda \M{I})^{-1}\theta^{\star} - \theta_{\star}}_{2} \\
        &=\norm{(\V{x}_{[a,b]}\Tra \V{x}_{[a,b]} + 
         \lambda \M{I})^{-1}\V{x}_{[a,b]}\Tra \epsilon
         - \lambda (\V{x}_{[a,b]}\Tra \V{x}_{[a,b]} + \lambda \M{I})^{-1}\theta_{\star}}_{2} \\
        &= \norm{(\V{x}_{[a,b]}\Tra \V{x}_{[a,b]} + \lambda \M{I})^{-1} (\V{x}\Tra \epsilon - \lambda \theta_{\star})}_{2}
        \\
        &\leq \norm{(\V{x}_{[a,b]}\Tra \V{x}_{[a,b]} + \lambda \M{I})^{-1}}_2 \norm{\V{x}\Tra \epsilon - \lambda \theta_{\star}}_2\\
        &\leq \frac{1}{s\phi^2_{\lambda, [a,b]}} \norm{\V{x}_{[a,b]}\Tra \epsilon - \lambda \theta_{\star}}_2.
    \end{aligned}
\end{equation}

We can expand $\epsilon\Tra \V{x}_{[a,b]}^{(r)} = \sum_{j \in [a ,b]} \epsilon(j) \V{X}_{j}^{(r)}$, where we let $\V{x}^{(r)}_{[a,b]}$ denote the $r^{\text{th}}$ column of $\V{x}_{[a,b]}$ and $D_{j, r} \equiv \epsilon(k) \V{x}_{j}^{(r)}$ is a $x_{\max} \sigma$-subgaussian random variable, conditioned on the sigma algebra $\textfrak{S}_{j-1}$ which is generated by random variables $X_1,..., X_{j-1}, Y_{1},..., Y_{j-1}$. Defining $D_{0, r} = 0$, the sequence $D_{0, r}, D_{1, r},..., D_{(b,r)}$ is a martingale difference sequence adapted to the filtration $\textfrak{S}_{1} \subset \textfrak{S}_{2} \subset ... \textfrak{S}_{b}$  since $E[\epsilon(j)X_{j}^{(r)}|\textfrak{S}_{j-1}] = 0$. Using Lemma \ref{lm:bernstein concentration},
\begin{equation}
    \begin{aligned}
    \label{eq:ridge case control prob}
        \text{Pr}&\left[ \norm{\hat{\theta}_{\lambda} - \theta_{\star}}_{2} \leq \chi \right] \\
        &\geq 
        1 - \sum_{r=1}^{d} \text{Pr}\left[ |\epsilon \Tra \V{X}_{[a,b]}^{(r)} - \lambda \theta^{(r)}_{\star}| > \frac{s
        \chi \phi^2_{\lambda, [a,b]} 
        }{\sqrt{d}}\right]  - \text{Pr}\left[ \lambda_{\min}(\hat{\Sigma}_{\lambda,[a,b]}) \leq \phi^2_{\lambda, [a,b]}) \right] \\
        &= 1 - (\sum_{r=1}^{d} \text{Pr}\left[ 
        \epsilon \Tra \V{x}_{[a,b]}^{(r)} > \underbrace{\lambda \theta^{(r)}_{\star} + \frac{s\chi \phi^2_{\lambda, [a,b]} }{\sqrt{d}}}_{\chi_{1,r}(\phi, \lambda,\theta^{\star})} \right] + 
        \text{Pr}\left[
        \epsilon \Tra \V{X}_{[a,b]}^{(r)} < \underbrace{\lambda \theta^{(r)}_{\star} - \frac{s\chi \phi^2_{\lambda, [a,b]} }{\sqrt{d}}}_{\chi_{2,r}(\phi, \lambda,\theta^{\star})}
        \right]),
    \end{aligned}    
\end{equation}
since the event $\lambda_{\min}(\hat{\Sigma}_{\lambda}(\V{x}_{[a,b]})) > \phi^2_{\lambda, [a,b]} > 0$ holds, $\text{Pr}\left[ \lambda_{\min}(\hat{\Sigma}_{\lambda,[a,b]}) \leq \phi^2_{\lambda, [a,b]}) \right] = 0$. 

Here we denote $\chi_{1,r}(\phi, \lambda,\theta^{\star}) = \lambda \theta^{(r)}_{\star} + \frac{s\chi \phi^2_{\lambda, [a,b]} }{\sqrt{d}}$ and $\chi_{2,r}(\phi, \lambda,\theta^{\star}) = \lambda \theta^{(r)}_{\star} - \frac{s\chi \phi^2_{\lambda, [a,b]} }{\sqrt{d}}$. By Lemma \ref{lm:bernstein concentration}, we can get the probability of this tail event,

\subsection{Bounds the first part of inequality \ref{eq:ridge case control prob}}
In the following, we give a brief case by case analysis to decompose these two tail events' probabilities.

\textbf{Case B.1.1}: \textit{If $\theta^{(r)}_{\star} = 0$, then $\chi_{1,r}(\phi, \lambda,\theta^{\star}) =  \frac{s\chi \phi^2_{\lambda, [a,b]} }{\sqrt{d}} > 0$. We have,
\begin{equation}
    \begin{aligned}
        \text{Pr}\left[ \epsilon \Tra X^{(r)} > \chi_{1,r}(\phi, \lambda,\theta^{\star}) \right] 
        &\leq \exp\left[-\frac{\chi_{1,r}^{2}(\phi, \lambda,\theta_{\star})}{2s  \norm{\V{x}_{[a,b]}}_{\infty}^{2}  \sigma^2}\right].
    \end{aligned}
\end{equation}
When $\lambda$ becomes smaller, $\chi_{1,r}^{2}(\phi, \lambda,\theta_{\star})$ becomes smaller.
Then the RHS exponential probability bound of  B.4 becomes larger. We always hope \ref{eq:ridge case control prob}'s probability bound smaller, then we can get a larger confidence ellipsoid. So $\lambda$ becoming smaller is our choice.}

\textbf{Case B.1.2}: \textit{If $\theta^{(r)}_{\star} > 0$, then $\chi_{1,r}(\phi, \lambda,\theta_{\star}) = \lambda \theta^{(r)}_{\star} + \frac{s\chi \phi^2_{\lambda, [a,b]} }{\sqrt{d}} > 0$. We have,
\begin{equation}
    \begin{aligned}
        \text{Pr}\left[ \epsilon \Tra X^{(r)} > \chi_{1,r}(\phi, \lambda,\theta_{\star}) \right] 
        &\leq \exp\left[-\frac{\chi_{1,r}^{2}(\phi, \lambda,\theta_{\star})}{2s \norm{\V{x}_{[a,b]}}_{\infty}^{2} \sigma^2}\right].
    \end{aligned}
\end{equation}
When $\lambda$ becomes smaller, $\chi^{2}_{1,r}(\phi, \lambda,\theta_{\star})$ becomes smaller. Then RHS exponential probability bound of B.5 becomes larger. Then part B.1.2's probability bound becomes smaller. We always hope the B.1.2's probability bound smaller, then we can get a larger confidence ellipsoid. So $\lambda$ becoming smaller is our choice.}

\textbf{Case B.1.3}: \textit{If $\theta^{(r)}_{\star} < 0$ and $\lambda < -\frac{s\chi \phi^2_{\lambda, [a,b]} }{\sqrt{d}\theta^{(r)}_{\star}}$, then $\chi_{1,r}(\phi, \lambda,\theta_{\star}) = \lambda \theta^{(r)}_{\star} + \frac{s\chi \phi^2_{\lambda, [a,b]} }{\sqrt{d}} > 0$. We have  
\begin{equation}
    \begin{aligned}
        \text{Pr}\left[ \epsilon \Tra X^{(r)} > \chi_{1,r}(\phi, \lambda,\theta_{\star}) \right] 
        &\leq \exp\left[-\frac{\chi_{1,r}^{2}(\phi, \lambda,\theta^{\star})}{2s \norm{\V{x}_{[a,b]}}_{\infty}^{2} \sigma^2}\right].
    \end{aligned}
\end{equation}
When $\lambda$ becomes larger, $\chi^{2}_{1,r}(\phi, \lambda,\theta_{\star})$ becomes smaller. Then RHS B.6's exponential probability bound becomes larger. Then part B.1.3's probability becomes smaller. We always hope the B.1.3's probability bound smaller, then we can get a larger confidence ellipsoid. Then the final probability gets smaller. So $\lambda$ becoming larger is our choice.}

\textbf{Case B.1.4}: \textit{If $\theta^{(r)}_{\star} < 0$ and $\lambda \geq  -\frac{s\chi \phi^2_{\lambda, [a,b]} }{\sqrt{d}\theta^{(r)}_{\star}}$, then $\chi_{1,r}(\phi, \lambda,\theta_{\star}) = \lambda \theta^{(r)}_{\star} + \frac{s\chi \phi^2_{\lambda, [a,b]} }{\sqrt{d}} < 0$, which means that this probability is larger than $\frac{1}{2}$ because our $\epsilon$ is symmetric random variable. We have
\begin{equation}
    \begin{aligned}
        \text{Pr}\left[ \epsilon \Tra X^{(r)} > \chi_{1,r}(\phi, \lambda, \theta_{\star}) \right].
    \end{aligned}
\end{equation}
If we want our confidence ellipsoid having a relative large probability, we need to avoid this case. So the choice for $\lambda$ is $\lambda <  -\frac{s\chi \phi^2_{\lambda, [a,b]} }{\sqrt{d}\theta^{(r)}_{\star}}$.}

When consider bounding the first part of \ref{eq:ridge case control prob}, $\lambda$ should be in the interval $\lambda \in [0, \underset{r \in  \mathcal{N}(\theta_{\star})}{\min}-\frac{s\chi \phi^2_{\lambda, [a,b]} }{\sqrt{d}\theta^{(r)}_{\star}})$. Then we just consider cases B.1.1, B.1.2 and B.1.3.

\subsection{Bounds the second part of inequality \ref{eq:ridge case control prob}}
In following cases. We analyze the second part of \ref{eq:ridge case control prob} and get the probability upper bound.

\textbf{Case B.2.1}: \textit{If $\theta^{(r)}_{\star} = 0$, then $\chi_{2,r}(\phi, \lambda,\theta_{\star}) = - \frac{s\chi \phi^2_{\lambda, [a,b]} }{\sqrt{d}} < 0$. We have,
\begin{equation}
    \begin{aligned}
        \text{Pr}\left[ \epsilon \Tra X^{(r)} < \chi_{2,r}(\phi, \lambda,\theta_{\star}) \right] 
        &\leq \exp\left[-\frac{\chi_{2,r}^{2}(\phi, \lambda, \theta_{\star})}{2s \norm{\V{x}_{[a,b]}}_{\infty}^{2} \sigma^2}\right]
    \end{aligned}.
\end{equation}
When $\lambda$ becomes smaller, $\chi_{1,r}^{2}(\phi, \lambda,\theta_{\star})$ becomes smaller.
Then RHS exponential probability bound of B.8 becomes larger. We always hope \ref{eq:ridge case control prob}'s probability bound smaller, then we can get a larger confidence ellipsoid. So $\lambda$ becoming smaller is our choice..}

\textbf{Case B.2.2}: \textit{If $\theta^{(r)}_{\star} < 0$, then $\chi_{2,r}(\phi, \lambda,\theta_{\star}) = \lambda \theta^{(r)}_{\star} - \frac{s\chi \phi^2_{\lambda, [a,b]} }{\sqrt{d}} < 0$. We have,
\begin{equation}
    \begin{aligned}
        \text{Pr}\left[ \epsilon \Tra X^{(r)} < \chi_{2,r}(\phi, \lambda,\theta_{\star}) \right] 
        &\leq \exp\left[-\frac{\chi_{2,r}^{2}(\phi, \lambda, \theta_{\star})}{2s \norm{\V{x}_{[a,b]}}_{\infty}^{2} \sigma^2}\right].
    \end{aligned}
\end{equation}
When $\lambda$ becomes larger, $|\chi_{2,r}(\phi, \lambda, \theta_{\star})|$ becomes larger. Then RHS exponential probability becomes smaller. 
Then part B.2.2's probability bound becomes smaller, we can get a larger confidence ellipsoid.
So $\lambda$ becoming larger is our choice.}

\textbf{Case B.2.3}: \textit{If $\theta^{(r)}_{\star} > 0$ and $\lambda < \frac{s\chi \phi^2_{\lambda, [a,b]} }{\sqrt{d}\theta^{(r)}_{\star}}$, then $\chi_{2,r}(\phi, \lambda,\theta_{\star}) = \lambda \theta^{(r)}_{\star} - \frac{s\chi \phi^2_{\lambda, [a,b]} }{\sqrt{d}} < 0$. We have 
\begin{equation}
    \begin{aligned}
        \text{Pr}\left[ \epsilon \Tra X^{(r)} < \chi_{2,r}(\phi, \lambda, \theta_{\star}) \right] 
        &\leq \exp\left[-\frac{\chi_{2,r}^{2}(\phi, \lambda,\theta_{\star})}{2s \norm{\V{x}_{[a,b]}}_{\infty}^{2} \sigma^2}\right].
    \end{aligned}
\end{equation}
When $\lambda$ becomes smaller, $|\chi_{2,r}(\phi, \lambda,\theta_{\star})|$ gets larger. Then RHS exponential probability bound becomes smaller. Then part B.2.3's probability bound becomes smaller, we can get a larger confidence ellipsoid. So $\lambda$ becoming smaller is our choice.}

\textbf{Case B.2.4}: \textit{If $\theta^{(r)}_{\star} > 0$ and $\lambda \geq \frac{s\chi \phi^2_{\lambda, [a,b]} }{\sqrt{d}\theta^{(r)}_{\star}}$, then $\chi_{2,r}(\phi, \lambda,\theta_{\star}) = \lambda \theta^{(r)}_{\star} - \frac{s\chi \phi^2_{\lambda, [a,b]} }{\sqrt{d}} >0$, which means that this probability is larger than $\frac{1}{2}$ because our $\epsilon$ is symmetric random variable. We have
\begin{equation}
    \begin{aligned}
        \text{Pr}\left[ \epsilon \Tra X^{(r)} < \chi_{2,r}(\phi, \lambda, \theta_{\star}) \right] 
        \geq \frac{1}{2}.
    \end{aligned}
\end{equation}
If we want our confidence ellipsoid having a relative large probability, we need to avoid this case. So the choice for
$\lambda < \frac{s\chi \phi^2_{\lambda, [a,b]} }{\sqrt{d}\theta^{(r)}_{\star}}$.}

When consider bounding the second part of \ref{eq:ridge case control prob}, $\lambda$ should be in the interval $\lambda \in [0, \underset{r \in  \mathcal{P}(\theta_{\star})}{\min}\frac{s\chi \phi^2_{\lambda, [a,b]} }{\sqrt{d}\theta^{(r)}_{\star}})$. Then we just consider cases B.2.1, B.2.2 and B.2.3.

\subsection{Lower bound of inequality \ref{eq:ridge case control prob}}
Combining $\lambda$ from subsections of B.2 and B.3, we get one adaptive interval for $\lambda$
\begin{equation}
  \lambda^{\text{bd}} = \min\{\underset{r \in  \mathcal{N}(\theta_{\star})}{\min}-\frac{s\chi \phi^2_{\lambda, [a,b]} }{\sqrt{d}\theta^{(r)}_{\star}}  ,\underset{r \in  \mathcal{P}(\theta_{\star})}{\min}\frac{s\chi \phi^2_{\lambda, [a,b]} }{\sqrt{d}\theta^{(r)}_{\star}}\} =  \frac{s\chi\phi^{2}_{[a,b]}}{\sqrt{d}} \frac{1}{\norm{\theta_{\star}}_{\infty}}.  
\end{equation}

Since cases B.1.2, B.1.3 and B.2.2, B.2.3 are two counteractive cases. If we know the number of positive coordinate of $\theta_{\star}$, $|\mathcal{P}(\theta_{\star})|$, is more than the number of negative coordinate of $\theta_{\star}$,  $|\mathcal{N}(\theta_{\star})|$. We would prefer cases B.1.3, B.2.3; otherwise, we prefer cases B.1.2, B.2.2. Therefore, it is a trade-off. 

Now let $d_{0} = d - |\mathcal{P}(\theta_{\star})| - |\mathcal{N}(\theta_{\star})|$ denotes the number of zero coordinate of $\theta_{\star}$. Then equation \eqref{eq:ridge case control prob} becomes
\begin{equation}
    \begin{aligned}
        \text{Pr}&\left[ \norm{\hat{\theta}_{\lambda,[a,b]} - \theta^{\star}}_{2} \leq \chi \right] \\
        &\geq 1 - (\sum_{r_{+} \in  \mathcal{P}(\theta_{\star})} \exp\left[-\frac{\chi_{1,r_{+}}^{2}(\phi, \lambda,\theta_{\star})}{2s \norm{\V{x}_{[a,b]}}_{\infty}^{2} \sigma^2}\right] 
        + \sum_{r_{-} \in  \mathcal{N}(\theta_{\star})} \exp\left[-\frac{\chi_{1,r_{-}}^{2}(\phi, \lambda,\theta_{\star})}{2s \norm{\V{x}_{[a,b]}}_{\infty}^{2} \sigma^2}\right] 
        + d_{0} \exp\left[-\frac{\chi_{1}^{2}(\phi, \lambda,0)}{2s \norm{\V{x}_{[a,b]}}_{\infty}^{2} \sigma^2}\right]) \\
        &- (\sum_{r_{+} \in  \mathcal{P}(\theta_{\star})} \exp\left[-\frac{\chi_{2,r_{+}}^{2}(\phi, \lambda,\theta_{\star})}{2s \norm{\V{x}_{[a,b]}}_{\infty}^{2} \sigma^2}\right] 
        +\sum_{r_{-} \in  \mathcal{N}(\theta_{\star})} \exp\left[-\frac{\chi_{2,r_{-}}^{2}(\phi, \lambda,\theta_{\star})}{2s \norm{\V{x}_{[a,b]}}_{\infty}^{2} \sigma^2}\right] 
        + d_{0} \exp\left[-\frac{\chi_{2}^{2}(\phi, \lambda,0)}{2s \norm{\V{x}_{[a,b]}}_{\infty}^{2} \sigma^2}\right]).
    \end{aligned}
\end{equation}
If we assume $\lambda \in [0, \lambda_{[a,b]}^{\text{bd}}]$, then 
\begin{equation}
    \begin{aligned}
    &= 1 
    - d_{0}(\exp\left[-\frac{\chi_{1}^{2}(\phi, \lambda,0)}{2s \norm{\V{x}_{[a,b]}}_{\infty}^{2} \sigma^2}\right] 
    + \exp\left[-\frac{\chi_{2}^{2}(\phi, \lambda,0)}{2s \norm{\V{x}_{[a,b]}}_{\infty}^{2} \sigma^2}\right]
    )  \\
    &\hspace{1cm}- (\sum_{r_{+} \in  \mathcal{P}(\theta_{\star})} \exp\left[-\frac{\chi_{1,r_{+}}^{2}(\phi, \lambda,\theta_{\star})}{2s \norm{\V{x}_{[a,b]}}_{\infty}^{2} \sigma^2}\right] + \exp\left[-\frac{\chi_{2,r_{+}}^{2}(\phi, \lambda,\theta_{\star})}{2s \norm{\V{x}_{[a,b]}}_{\infty}^{2} \sigma^2}\right]
    ) \\
    &\hspace{1cm}- (\sum_{r_{-} \in  \mathcal{N}(\theta_{\star})} \exp\left[-\frac{\chi_{1,r_{-}}^{2}(\phi, \lambda,\theta_{\star})}{2s \norm{\V{x}_{[a,b]}}_{\infty}^{2} \sigma^2}\right] + \exp\left[-\frac{\chi_{2,r_{-}}^{2}(\phi, \lambda,\theta_{\star})}{2s \norm{\V{x}_{[a,b]}}_{\infty}^{2} \sigma^2}\right]).
    \end{aligned}
\end{equation}
Since we know the $\chi_{1}^{2}(\phi, \lambda,0) = \chi_{2}^{2}(\phi, \lambda,0) = \frac{s^2\chi^2 \phi^4_{\lambda, [a,b]} }{d}$. Thus 
\begin{equation}
    \begin{aligned}
    \label{eq:ridge-part-1}
    \exp\left[-\frac{\chi_{1}^{2}(\phi, \lambda,0)}{2s \norm{\V{x}_{[a,b]}}_{\infty}^{2} \sigma^2}\right] + \exp\left[-\frac{\chi_{2}^{2}(\phi, \lambda,0)}{2s \norm{\V{x}_{[a,b]}}_{\infty}^{2} \sigma^2}\right]
    ) = 2 \exp(-\frac{s\chi^2 \phi^4_{\lambda,[a,b]} }{2 d  \norm{\V{x}_{[a,b]}}_{\infty}^{2}  \sigma^2}).
    \end{aligned}
\end{equation}

Then the second positive coordinate of $\theta_{\star}$ part becomes
\begin{equation}
    \begin{aligned}
        &= \sum_{r_{+} \in  \mathcal{P}(\theta_{\star})} \exp\left[-\frac{\chi_{1,r_{+}}^{2}(\phi, \lambda,\theta_{\star})}{2s \norm{\V{x}_{[a,b]}}_{\infty}^{2} \sigma^2}\right] + \exp\left[-\frac{\chi_{2,r_{+}}^{2}(\phi, \lambda,\theta_{\star})}{2s \norm{\V{x}_{[a,b]}}_{\infty}^{2} \sigma^2}\right]
        \\
        &=\sum_{r_{+} \in  \mathcal{P}(\theta_{\star})} \exp\left[-\frac{
        (\lambda\theta_{\star}^{r_{+}} +\frac{s\chi \phi^{2}_{\lambda, [a,b]}}{\sqrt{d}})^{2}
        }{2s \norm{\V{x}_{[a,b]}}_{\infty}^{2} \sigma^2}\right] + \exp\left[-\frac{
        (\lambda\theta_{\star}^{r_{+}} - \frac{s\chi \phi^{2}_{\lambda, [a,b]}}{\sqrt{d}})^{2}
        }{2s \norm{\V{x}_{[a,b]}}_{\infty}^{2} \sigma^2}\right].
    \end{aligned}
\end{equation}
Since we want to get upper bound of the above equation and then to get the confidence interval,  we maximizes this exponential value by selecting the minimum of $\chi_{1, r_{+}}^{2} \text{and} \chi_{2, r_{+}}^{2}$. 

Since in the case B.1.2, we know $\chi_{1,r_{+}}(\phi, \lambda,\theta_{\star}) = \lambda \theta_{\star}^{r_{+}} +\frac{s\chi \phi^{2}_{\lambda, [a,b]}}{\sqrt{d}} > 0$ and by selecting the $\underset{r_{+} \in \mathcal{P}(\theta_{\star}) }{\min} \theta_{\star}^{r_{+}}$ given fixed $\lambda$. We denote the minimum positive coordinate of $\theta_{\star}$ as $\mathcal{P}_{\text{min}}(\theta_{\star})$, $C_{1}(\phi_{\lambda, [a,b]}, \chi, d) = \frac{s\chi\phi_{\lambda, [a,b]}^{2}}{\sqrt{d}}$; In the second part of the exponential, we know in case B.2.3, $\chi_{2,r_{+}}(\phi, \lambda,\theta_{\star}) = \lambda \theta_{\star}^{r_{+}} - \frac{s\chi \phi^{2}_{\lambda, [a,b]}}{\sqrt{d}} < 0$.  So if we want to minimize $|\chi_{2, r_{+}}|$ given fixed lambda, select the maximum: $\underset{r_{+}\in \mathcal{P}(\theta_{\star})}{\max} \theta_{\star}^{r_{+}}$. We denote the maximum positive coordinate of $\theta_{\star}$ as $\mathcal{P}_{\text{max}}(\theta_{\star})$.
Thus, taking both of these into consideration to get an upper bound of this summation of probability, we can obtain
\begin{equation}
    \begin{aligned}
        &\leq \sum_{r_{+} \in  \mathcal{P}(\theta_{\star})} \exp\left[-\frac{
        (\lambda\mathcal{P}_{\text{min}}(\theta_{\star}) +C_{1}(\phi_{\lambda, [a,b]}, \chi, d))^{2}
        }{2s \norm{\V{x}_{[a,b]}}_{\infty}^{2} \sigma^2}\right] + \exp\left[-\frac{
        (\lambda\mathcal{P}_{\text{max}}(\theta_{\star}) - C_{1}(\phi_{\lambda, [a,b]}, \chi, d))^{2}
        }{2s \norm{\V{x}_{[a,b]}}_{\infty}^{2} \sigma^2}\right] \\
        &=|\mathcal{P}(\theta_{\star})|(\exp\left[-\frac{
        (\lambda\mathcal{P}_{\text{min}}(\theta_{\star}) +C_{1}(\phi_{\lambda, [a,b]}, \chi, d))^{2}
        }{2s \norm{\V{x}_{[a,b]}}_{\infty}^{2} \sigma^2}\right] + \exp\left[-\frac{
        (\lambda \mathcal{P}_{\text{max}}(\theta_{\star}) - C_{1}(\phi_{\lambda, [a,b]}, \chi, d))^{2}
        }{2s \norm{\V{x}_{[a,b]}}_{\infty}^{2} \sigma^2}\right]) \\
        &\leq 2|\mathcal{P}(\theta_{\star})| \max \bigg\{ \exp\left[-\frac{
        (\lambda\mathcal{P}_{\text{min}}(\theta_{\star}) +C_{1}(\phi_{\lambda, [a,b]}, \chi, d))^{2}
        }{2s \norm{\V{x}_{[a,b]}}_{\infty}^{2} \sigma^2}\right], 
        \exp\left[-\frac{
        (\lambda \mathcal{P}_{\text{max}}(\theta_{\star}) - C_{1}(\phi_{\lambda, [a,b]}, \chi, d))^{2}
        }{2s \norm{\V{x}_{[a,b]}}_{\infty}^{2} \sigma^2}\right]\bigg\},
    \end{aligned}
\end{equation}
where constant $C_{1}(\phi_{\lambda, [a,b]}, \chi, d) = \frac{s\chi \phi^{2}_{\lambda, [a,b]}}{\sqrt{d}}$.

If $\lambda \leq \frac{2C_{1}(\phi_{\lambda,[a,b]}, \chi, d)}{\mathcal{P}_{\text{max}}(\theta_{\star}) -\mathcal{N}_{\text{max}}(\theta_{\star})}$, then (B.17)
\begin{equation}
    \begin{aligned}
    \label{eq:ridge-part-2}
    \leq 2|\mathcal{P}(\theta_{\star})| \exp\left[-\frac{
        (\lambda \mathcal{P}_{\text{min}}(\theta_{\star}) +C_{1}(\phi_{\lambda, [a,b]}, \chi, d))^{2}
        }{2s \norm{\V{x}_{[a,b]}}_{\infty}^{2} \sigma^2}\right]
    \end{aligned}
\end{equation}
This condition is easily to be satisfied since $-\mathcal{N}_{\text{max}}(\theta_{\star})$ should be a very small value.

The third negative part becomes
\begin{equation}
    \begin{aligned}
        &= \sum_{r_{-} \in  \mathcal{N}(\theta_{\star})} \exp\left[-\frac{\chi_{1,r_{-}}^{2}(\phi, \lambda,\theta^{\star})}{2s \norm{\V{x}_{[a,b]}}_{\infty}^{2} \sigma^2}\right] + \exp\left[-\frac{\chi_{2,r_{-}}^{2}(\phi, \lambda,\theta^{\star})}{2s \norm{\V{x}_{[a,b]}}_{\infty}^{2} \sigma^2}\right]
        \\
        &=\sum_{r_{-} \in  \mathcal{N}(\theta_{\star})} \exp\left[-\frac{
        (\lambda\theta_{\star}^{r_{-}} +\frac{s\chi \phi^{2}_{\lambda, [a,b]}}{\sqrt{d}})^{2}
        }{2s \norm{\V{x}_{[a,b]}}_{\infty}^{2} \sigma^2}\right] + \exp\left[-\frac{
        (\lambda\theta_{\star}^{r_{-}} - \frac{s\chi \phi^{2}_{\lambda, [a,b]}}{\sqrt{d}})^{2}
        }{2s \norm{\V{x}_{[a,b]}}_{\infty}^{2} \sigma^2}\right].
    \end{aligned}
\end{equation}
Since we want to get an upper bound of the above equation and then to get the confidence interval, we maximizes this exponential value by selecting the minimum of $\chi_{1, r_{-}}^{2} \text{and} \chi_{2, r_{-}}^{2}$. 

In the case B.1.3, we know $ \chi_{1,r_{-}}(\phi, \lambda,\theta_{\star}) = \lambda \theta_{\star}^{r_{-}} +\frac{s\chi \phi^{2}_{\lambda, [a,b]}}{\sqrt{d}}) > 0$ and by selecting  $\underset{r_{-}}{\min} \theta_{\star}^{r_{-}}$ given fixed $\lambda$. We denote the minimum negative coordinate of $\theta_{\star}$
as $\mathcal{N}_{\text{min}}(\theta_{\star})$;

In the second part of the exponential, we know in case B.2.2, $ \chi_{2,r_{-}}(\phi, \lambda,\theta_{\star}) = (\lambda_{[a,b]}\theta^{\star}_{r_{-}} - \frac{s\chi \phi^{2}_{\lambda, [a,b]}}{\sqrt{d}}) < 0$, so if we want to minimize $|\chi_{2, r_{-}}|$, we need to select the maximum: $\underset{r_{-}\in \mathcal{N}(\theta_{\star}) }{\max} \theta_{\star}^{r_{-}}$. We denote the maximum negative coordinate of $\theta_{\star}$ as $\mathcal{N}_{\text{max}}(\theta_{\star})$.

Thus, taking both of these into consideration to get an upper bound of this summation of probability, we can obtain
\begin{equation}
    \begin{aligned}
        &\leq \sum_{r_{-} \in  \mathcal{N}(\theta_{\star})} \exp\left[-\frac{
        (\lambda \mathcal{N}_{\text{min}}(\theta_{\star}) +C_{1}(\phi_{\lambda, [a,b]}, \chi, d))^{2}
        }{2s \norm{\V{x}_{[a,b]}}_{\infty}^{2} \sigma^2}\right] + \exp\left[-\frac{
        (\lambda \mathcal{N}_{\text{max}}(\theta_{\star}) - C_{1}(\phi_{\lambda, [a,b]}, \chi, d))^{2}
        }{2s \norm{\V{x}_{[a,b]}}_{\infty}^{2} \sigma^2}\right] \\
        &=|\mathcal{N}(\theta_{\star})|(\exp\left[-\frac{
        (\lambda \mathcal{N}_{\text{min}}(\theta_{\star}) +C_{1}(\phi_{\lambda, [a,b]}, \chi, d))^{2}
        }{2s \norm{\V{x}_{[a,b]}}_{\infty}^{2} \sigma^2}\right] + \exp\left[-\frac{
        (\lambda \mathcal{N}_{\text{max}}(\theta_{\star}) - C_{1}(\phi_{\lambda, [a,b]}, \chi, d))^{2}
        }{2s \norm{\V{x}_{[a,b]}}_{\infty}^{2} \sigma^2}\right]) \\
        &\leq 2|\mathcal{N}(\theta_{\star})| \max \bigg\{ \exp\left[-\frac{
        (\lambda \mathcal{N}_{\text{min}}(\theta_{\star}) +C_{1}(\phi_{\lambda, [a,b]}, \chi, d))^{2}
        }{2s \norm{\V{x}_{[a,b]}}_{\infty}^{2} \sigma^2}\right], 
        \exp\left[-\frac{
        (\lambda \mathcal{N}_{\text{max}}(\theta_{\star}) - C_{1}(\phi_{\lambda, [a,b]}, \chi, d))^{2}
        }{2s \norm{\V{x}_{[a,b]}}_{\infty}^{2} \sigma^2}\right]\bigg\}. 
    \end{aligned}
\end{equation}

If $\lambda  \leq \frac{2C_{1}(\phi_{\lambda, [a,b]}, \chi, d)}{\mathcal{N}_{\text{max}}(\theta_{\star}) - \mathcal{N}_{\text{min}}(\theta_{\star})}$, then 
\begin{equation}
    \begin{aligned}
    \label{eq:ridge-part-3}
    \leq 2|\mathcal{N}(\theta_{\star})| \exp\left[-\frac{
        (\lambda \mathcal{N}_{\text{max}}(\theta_{\star}) - C_{1}(\phi_{\lambda, [a,b]}, \chi, d))^{2}
        }{2s \norm{\V{x}_{[a,b]}}_{\infty}^{2} \sigma^2}\right]
    \end{aligned}
\end{equation}
This is similar to the previous one in equation \eqref{eq:ridge-part-2}.

Finally, when we combine equations \eqref{eq:ridge-part-1}, \eqref{eq:ridge-part-2}, \eqref{eq:ridge-part-3} together, we get the estimated probability 
\begin{equation}
    \begin{aligned}
    \text{Pr}&\left[ \norm{\hat{\theta}_{\lambda,[a,b]} - \theta^{\star}}_{2} \leq \chi \right] \\
    &\geq
    1 - 2d_{0} \exp\left[-\frac{C_{1}^{2}(\phi_{\lambda, [a,b]}, \chi, d) }{2s \norm{\V{x}_{[a,b]}}_{\infty}^{2} \sigma^2}\right] 
    - 2|\mathcal{P}(\theta_{\star})| \exp\left[-\frac{
        (\lambda \mathcal{P}_{\text{min}}(\theta_{\star}) +C_{1}(\phi_{\lambda, [a,b]}, \chi, d))^{2}
        }{2s \norm{\V{x}_{[a,b]}}_{\infty}^{2} \sigma^2}\right]\\
    &\hspace{0.6cm} -2|\mathcal{N}(\theta_{\star})| \exp\left[-\frac{
        (\lambda \mathcal{N}_{\text{max}}(\theta_{\star}) - C_{1}(\phi_{\lambda, [a,b]}, \chi, d))^{2}
        }{2s \norm{\V{x}_{[a,b]}}_{\infty}^{2} \sigma^2}\right].
    \end{aligned}
\end{equation}
Since we want to control the confidence set with probability at least $1-\delta$, let $\text{Pr}\left[ \norm{\hat{\theta}_{\lambda,[a,b]} - \theta^{\star}}_{2} \leq \chi \right] \geq 1-\delta$, we have 
\begin{equation}
    \begin{aligned}
       \underbrace{ d\exp\left[-\frac{C_{1}^{2}(\phi_{\lambda, [a,b]}, \chi, d) }{2s \norm{\V{x}_{[a,b]}}_{\infty}^{2} \sigma^2}\right]}_{\text{Part I}}
        &+ \underbrace{|\mathcal{P}(\theta_{\star})| \exp\left[-\frac{
        \lambda^2(\mathcal{P}_{\text{min}}(\theta_{\star}))^{2}         +2\lambda \mathcal{P}_{\text{min}}(\theta_{\star}) C_{1}(\phi_{\lambda, [a,b]}, \chi, d)
        }{2s \norm{\V{x}_{[a,b]}}_{\infty}^{2} \sigma^2}\right]}_{\text{Part II}}\\
        &+ \underbrace{|\mathcal{N}(\theta_{\star})| \exp\left[-\frac{
        \lambda^{2}(\mathcal{N}_{\text{max}}(\theta_{\star}))^{2} 
        - 2\lambda \mathcal{N}_{\text{max}}(\theta_{\star}) C_{1}(\phi_{\lambda, [a,b]}, \chi, d)
        }{2s \norm{\V{x}_{[a,b]}}_{\infty}^{2} \sigma^2}\right]}_{\text{Part III}}
        \leq \frac{\delta}{2}.
    \end{aligned}
\end{equation}
In the following, we present the assumption we need to make the above inequality to have an analytic solution.

\textbf{Assumption 1}. 
\textit{(Weakest Positive to Strongest Signal Ratio) This is the condition for considering the case that positive coordinate of $\theta_{\star}$ dominates the bad events happening and without loss of generality, we assume that}
\begin{equation}
    \label{Appendix: WPS ratio}
    \text{WPSSR} = \frac{\mathcal{P}_{\text{min}}(\theta_{\star})}
    {\norm{\theta_{\star}}_{\infty}}\leq \frac{-\sqrt{\log{\frac{6d}{\delta}}\log{\frac{2d}{\delta}}} + \sqrt{\log{\frac{6d}{\delta}}\log{\frac{2d}{\delta}}+ s^{2} \log{\frac{2d}{\delta}}\log{\frac{12|\mathcal{P}(\theta_{\star})|}{\delta}}}}{s\log{\frac{2d}{\delta}}}.
\end{equation}
\textbf{Remarks}. The WPSSR is monotone increasing in $s$ and $\mathcal{P}(\theta_{\star})$, and is monotone decreasing in $d$. However, as long as $s\geq d$, in most cases, the LHS is greater than one. For example, if $s=100, d=110, \delta=0.05, |\mathcal{P}(\theta_{\star})| = 30$, then $\text{WPSSR}$ needs to be less than 1.02, which is satisfied automatically.

First we have a weak assumption that \textit{Part I} is always less than $\delta/2$, with this assumption, we can have a initial interval for $\chi(\delta)$,
\begin{equation}
    \begin{aligned}
        \chi(\delta)  \geq \sqrt{2d}\sigma
            \frac{x_{\max, [a,b]}}{\phi^2_{\lambda, [a,b]}\sqrt{b-a}}
            \sqrt{\log(\frac{2d}{\delta})}.
    \end{aligned}
\end{equation}
Then we plug in this initial $\chi(\delta)$ into the $\lambda^{\text{bd}}$, we can get the initial interval for $\lambda$,
\begin{equation}
    \begin{aligned}
        \lambda \leq 
        \sigma \norm{\V{x}_{[t-1,t-s]}}_{\infty} \sqrt{2s \log(2d/\delta)}  /\norm{\theta_{\star}}_{\infty}.
    \end{aligned}
\end{equation}

To get an analytical confidence ellipsoid, without loss of generality,  we assume \textit{Part II} is greater than \textit{Part III} and \textit{Part I} with assumption 1 and if select $\lambda$ following Lemma \ref{eq: Ridge hyperparameter selection}, then
\begin{equation}
    \begin{aligned}
        3|\mathcal{P}(\theta_{\star})| \exp\left[-\frac{
        \lambda^2(\mathcal{P}_{\text{min}}(\theta_{\star}))^{2}         +2\lambda \mathcal{P}_{\text{min}}(\theta_{\star}) C_{1}(\phi_{\lambda, [a,b]}, \chi, d)
        }{2s \norm{\V{x}_{[a,b]}}_{\infty}^{2} \sigma^2}\right] \leq \frac{\delta}{2}.
    \end{aligned}
\end{equation}

Thus, with confidence $1-\delta$, we have the following confidence ellipsoid for \texttt{FIFD-Adaptive Ridge} method, 
\begin{equation}
    \begin{aligned}
        C_{\lambda, [t-s,t-1]} = \bigg\{\theta \in \Real^{d}: 
        \norm{\hat{\theta}_{[t-s,t-1]} - \theta}_{\Phi_{\lambda, [t-s,t-1]}}
        &\leq 
        \sigma \kappa \nu
        q_{\lambda, [t-s,t-1]} 
        \sqrt{d/2s}
        \bigg\},
        \end{aligned}    
    \end{equation}
where $q_{\lambda, [t-s,t-1]} = \norm{\V{x}_{[t-1,t-s]}}_{\infty}/ \phi^2_{\lambda, [t-s,t-1]}$, $\kappa = \sqrt{\log^{2}(6|\mathcal{P}(\theta_{\star})|/ \delta)/\log(2d/ \delta)}$, and $\nu =\norm{\theta_{\star}}_{\infty}/ \mathcal{P}_{\text{min}}(\theta_{\star})$.

\end{proof}

\section{FIFD-OLS Regret}
\label{Appendix-FIFD-OLS Regret}

\begin{theorem}
\textit{(Regret Upper Bound of The FIFD-OLS Algorithm). Assume that for all $t \in [s+1, T-s]$ and ${X_{t}}$ is i.i.d random variables with distribution $\mathcal{P}_{\mathcal{X}}$. With probability at least $1-\delta \in [0,1]$ and Lemma \ref{lm: Appendix OLS CE} holds, for all $T>s, s\geq d$, we have an upper bound on the cumulative regret at time $T$:
\begin{equation}
    \begin{aligned}
        R_{T, s}(\mathcal{A}_{OLS}) \leq 
            2\sigma \zeta \sqrt{(d/s) \log(2d/\delta)(T-s)
            \left( d \log (sL^{2}/d) +(T-s) \right)}, 
    \end{aligned}
\end{equation}
where the adaptive constant $\zeta = \underset{s+1 \leq t\leq T}{\mathrm{max}} q_{[t-s,t-1]}$.}
\end{theorem}

\begin{proof}
Here we use $l_{t}$ to denote the instantaneous \textit{absolute loss} at time $t$.
Let's decompose the instantaneous absolute loss as follows:
\begin{equation}
    \begin{aligned}
        l_{t} &= |\langle \hat{\theta}_{[t-s,t-1]}, X_{t} \rangle - \langle \theta_{\star}, X_{t} \rangle | \\
        & = |\langle \hat{\theta}_{[t-s,t-1]} - \theta_{\star}, x_{t} \rangle| \\
        & = |[\hat{\theta}_{[t-s,t-1]} - \theta_{\star}]\Tra x_{t}| \\
        & = |[\hat{\theta}_{[t-s,t-1]} - \theta_{\star}]\Tra \Phi_{[t-s,t-1]}^{\frac{1}{2}} \Phi_{[t-s,t-1]}^{-\frac{1}{2}} x_{t}| \\
        & = |[\hat{\theta}_{[t-s,t-1]} - \theta_{\star}]\Tra \Phi_{[t-s,t-1]}^{\frac{1}{2}}| \times |\Phi_{[t-s,t-1]}^{-\frac{1}{2}} x_{t}| \\
        & \leq \norm{\hat{\theta}_{[t-s,t-1]} - \theta_{\star}}_{\Phi_{[t-s,t-1]}} \norm{x_{t}}_{\Phi_{[t-s,t-1]}^{-1}} \\
        & \leq \sqrt{\beta_{[t-s,t-1]}(\delta)}  \norm{x_{t}}_{\Phi_{[t-s,t-1]}^{-1}},
    \end{aligned}
\end{equation}
where the last step from Lemma \ref{lm:OLS Confidence Ellipsoid}. 
Thus, with probability at least $1-\delta$, for all $T>s$,
\begin{equation}
    \begin{aligned}
        R_{T, s}(\mathcal{A}) &= \sqrt{(T-s)\sum_{t = s+1}^{T} r_{t}^2} \\
        & \leq 
            \sqrt{(T-s) \sum_{t = s+1}^{T}  \beta_{[t-s,t-1]}(\delta)\norm{x_{t}}_{\Phi_{[t-s,t-1]}^{-1}}^{2}},
    \end{aligned}
\end{equation}
where the last step we use Lemma \ref{lem: LRT-lemma} to process the deletion and addition procedure. So we have 
\begin{equation}
    \begin{aligned}
        & \leq 
            \sqrt{2(T-s)
            \underset{s+1 \leq t\leq T}{\mathrm{max}} \beta_{[t-s-1,t-1]}(\delta) 
            \left( d \log (\frac{sL^{2}}{d}) +(T-s) \right) 
            }\\
        & \leq 2\sigma \zeta \sqrt{(d/s) \log(2d/\delta)(T-s)
            \left( d \log (sL^{2}/d) +(T-s) \right)},
    \end{aligned}
\end{equation}
where the last step uses the confidence ellipsoid from Lemma \ref{lm:OLS Confidence Ellipsoid}.
\end{proof}

\begin{lemma}
The cumulative regret of FIFD-OLS is partially determined by FRT at each time step,
\begin{equation}
    \begin{aligned}
    \sum_{t=s+1}^{T} \norm{x_{t}}_{\Phi_{[t-s,t-1]}^{-1}}^{2} 
    \leq
        2\eta_{\text{OLS}}
        + 
        \sum_{t =s+1}^{T} 
        \text{FRT}_{[t-s, t-1]}
    \leq        
        2 \left[d\log(\frac{sL^{2}}{d})
        + 
        (T-s)\right]
    \end{aligned}
\end{equation}
where $\eta_{\text{OLS}} = \log(\text{det}(\Phi_{[T-s, T]}))$ is a constant based on data time window$[T-s, T]$.
\end{lemma}

\begin{proof}
Here we use $a = t-s, b = t-1$ for the sake of simplicity. Elementary algebra gives
\begin{equation}
    \begin{aligned}
        \text{det}&(\Phi_{[a+1,b+1]})\\ 
        & = \text{det}(\Phi_{\lambda, [a,b]} + x_{b+1}x_{b+1}\Tra - x_{a}x_{a}\Tra)  \\
        & = \text{det}(\Phi_{\lambda, [a,b]}^{1/2}
        (\M{I} + \Phi_{\lambda, [a,b]}^{-1/2}(x_{b+1}x_{b+1}\Tra - x_{a}x_{a}\Tra)\Phi_{\lambda, [a,b]}^{-1/2})
        \Phi_{\lambda, [a,b]}^{1/2}) \\
        & = \text{det}(\Phi_{\lambda, [a,b]})\text{det}(\M{I} + \Phi_{\lambda, [a,b]}^{-1/2}(x_{b+1}x_{b+1}\Tra - x_{a}x_{a}\Tra)\Phi_{\lambda, [a,b]}^{-1/2}) \\
        & = \text{det}(\Phi_{\lambda, [a,b]})\text{det}(\M{I} + \Phi_{\lambda, [a,b]}^{-1/2}x_{b+1}x_{b+1}\Tra \Phi_{\lambda, [a,b]}^{-1/2} - 
        \Phi_{\lambda, [a,b]}^{-1/2}x_{a}x_{a}\Tra \Phi_{\lambda, [a,b]}^{-1/2}) \\
        & = \text{det}(\Phi_{\lambda, [a,b]}) \text{det}\left(\M{I} + (\Phi_{\lambda, [a,b]}^{-1/2} x_{b+1})(\Phi_{\lambda, [a,b]}^{-1/2} x_{b+1})\Tra - 
        (\Phi_{\lambda, [a,b]}^{-1/2} x_{a})(\Phi_{\lambda, [a,b]}^{-1/2} x_{a})\Tra \right) \\
        & = \text{det}(\Phi_{\lambda, [a,b]}) \left((1 + \norm{\Phi_{\lambda, [a,b]}^{-1/2} x_{b+1}}^{2})(1 - \norm{\Phi_{\lambda, [a,b]}^{-1/2} x_{a}}^{2}) + \langle \Phi_{\lambda, [a,b]}^{-1/2} x_{b+1}, \Phi_{\lambda, [a,b]}^{-1/2} x_{a} \rangle^{2}\right).
    \end{aligned}
\end{equation}
where the last step use lemma \ref{lm:det},
\begin{equation}
\label{eq: det cos c.7}
    \begin{aligned}
        & = \text{det}(\Phi_{\lambda, [a,b]}) 
        \left( 
        (1 + \norm{x_{b+1}}_{\Phi_{\lambda, [a,b]}^{-1}}^{2})
        (1 - \norm{x_{a}}_{\Phi_{\lambda, [a,b]}^{-1}}^{2})
        + \langle x_{b+1}, x_{a} \rangle_{\Phi_{\lambda, [a,b]}^{-1}}^{2}
        \right) \\
        & = \text{det}(\Phi_{\lambda, [a,b]})  
        \left( 1 + \norm{x_{b+1}}_{\Phi_{\lambda, [a,b]}^{-1}}^{2} - \norm{x_{a}}_{\Phi_{\lambda, [a,b]}^{-1}}^{2}  +
        \langle x_{b+1}, x_{a} \rangle_{\Phi_{\lambda, [a,b]}^{-1}}^{2}
        - \norm{x_{b+1}}_{\Phi_{\lambda, [a,b]}^{-1}}^{2} \norm{x_{a}}_{\Phi_{\lambda, [a,b]}^{-1}}^{2} \right)\\
        & = \text{det}(\Phi_{\lambda, [a,b]})  
        \left(
        1 + \norm{x_{b+1}}_{\Phi_{\lambda, [a,b]}^{-1}}^{2} - \norm{x_{a}}_{\Phi_{\lambda, [a,b]}^{-1}}^{2} + 
        \norm{x_{b+1}}_{\Phi_{\lambda, [a,b]}^{-1}}^{2} \norm{x_{a}}_{\Phi_{\lambda, [a,b]}^{-1}}^{2}
        (\cos^{2}_{\Phi_{\lambda, [a,b]}^{-1}}{\theta} - 1)
        \right).
    \end{aligned}
\end{equation}
where $\cos_{\Phi_{\lambda, [a,b]}^{-1}}{\theta} = \frac{\langle x_{b+1}, x_{a} \rangle_{\Phi_{\lambda, [a,b]}^{-1}}}{\norm{x_{b+1}}_{\Phi_{\lambda, [a,b]}^{-1}} \norm{x_{a}}_{\Phi_{\lambda, [a,b]}^{-1}}}$, which measures the similarity of the vector of $x_{b+1}$ and $x_{a}$ with respect to $\Phi_{\lambda, [a,b]}^{-1}$. If $\cos^{2}_{\Phi_{\lambda, [a,b]}^{-1}}{\theta} = 1$, that means the incoming data $x_{b+1}$ and the deleted data $x_{a}$ are same. However, if $\cos^{2}_{\Phi_{\lambda, [a,b]}^{-1}}{\theta} = 0$, that means the incoming data and the deleted data are totally different with respect to $\Phi_{\lambda, [a,b]}^{-1}$. 

Now we switch to the notation $t$, where $t-s = a, t-1 =b$. At time step $T$ and combining equation \eqref{eq: det cos c.7}, we have 
\begin{equation}
\label{eq: C.8}
    \begin{aligned}
        \text{det}(\Phi_{[T-s, T]}) 
         = \prod_{t=s+1}^{T} 
        (&
        1 + \norm{x_{t}}_{\Phi_{[t-s,t-1]}^{-1}}^{2} 
        - \norm{x_{t-s}}_{\Phi_{[t-s,t-1]}^{-1}}^{2} \\ 
        & + \norm{x_{t}}_{\Phi_{[t-s,t-1]}^{-1}}^{2} \norm{x_{t-s}}_{\Phi_{[t-s,t-1]}^{-1}}^{2}
        (\cos^{2}_{\Phi_{[t-s,t-1]}^{-1}}{\theta} - 1)
        ).
    \end{aligned}
\end{equation}

Taking $\log$ to both side of equation \eqref{eq: C.8},  we get
\begin{equation}
    \begin{aligned}
        \log(\text{det}(\Phi_{[T-s, T]})) 
         = \sum_{t=s+1}^{T}
        \log 
        (
        &1 + \norm{x_{t}}_{\Phi_{[t-s,t-1]}^{-1}}^{2} 
        - \norm{x_{t-s}}_{\Phi_{[t-s,t-1]}^{-1}}^{2} \\ 
        & + \norm{x_{t}}_{\Phi_{[t-s,t-1]}^{-1}}^{2} \norm{x_{t-s}}_{\Phi_{[t-s,t-1]}^{-1}}^{2}
        (\cos^{2}_{\Phi_{[t-s,t-1]}^{-1}}{\theta} - 1)
        ).
    \end{aligned}
\end{equation}
Combining $\log(1+x) > \frac{x}{1+x}$ which holds when $x > -1$,
we first consider each part of the product and use the dissimilarity measure $\sin^{2}_{\Phi_{\lambda, [a,b]}^{-1}}{\theta} =  1- \cos^{2}_{\Phi_{\lambda, [a,b]}^{-1}}{\theta}$. So we get 
\begin{equation}
    \begin{aligned}
        \log &
        \left(
        1 + \norm{x_{t}}_{\Phi_{[t-s,t-1]}^{-1}}^{2} - \norm{x_{t-s}}_{\Phi_{[t-s,t-1]}^{-1}}^{2} - \norm{x_{t}}_{\Phi_{[t-s,t-1]}^{-1}}^{2} \norm{x_{t-s}}_{\Phi_{[t-s,t-1]}^{-1}}^{2}
        \sin^{2}_{\Phi_{[t-s,t-1]}^{-1}}{\theta}
        \right) \\
        &> 
        \frac{\norm{x_{t}}_{\Phi_{[t-s,t-1]}^{-1}}^{2} - \norm{x_{t-s}}_{\Phi_{[t-s,t-1]}^{-1}}^{2} - \norm{x_{t}}_{\Phi_{[t-s,t-1]}^{-1}}^{2} \norm{x_{t-s}}_{\Phi_{[t-s,t-1]}^{-1}}^{2}
        \sin^{2}_{\Phi_{[t-s,t-1]}^{-1}}{\theta}
        }
        {1 
        + \norm{x_{t}}_{\Phi_{[t-s,t-1]}^{-1}}^{2} 
        - \norm{x_{t-s}}_{\Phi_{[t-s,t-1]}^{-1}}^{2}
        - \norm{x_{t}}_{\Phi_{[t-s,t-1]}^{-1}}^{2} \norm{x_{t-s}}_{\Phi_{[t-s,t-1]}^{-1}}^{2}
        \sin^{2}_{\Phi_{[t-s,t-1]}^{-1}}{\theta}
        } \\
        &> \frac{1}{2} 
        \left[\norm{x_{t}}_{\Phi_{[t-s,t-1]}^{-1}}^{2} - \norm{x_{t-s}}_{\Phi_{[t-s,t-1]}^{-1}}^{2} - \norm{x_{t}}_{\Phi_{[t-s,t-1]}^{-1}}^{2} \norm{x_{t-s}}_{\Phi_{[t-s,t-1]}^{-1}}^{2}
        \sin^{2}_{\Phi_{[t-s,t-1]}^{-1}}{\theta}
        \right].
    \end{aligned}
\end{equation}
Therefore, we can give a bound of $\sum_{t=s+1}^{T} \norm{x_{t}}_{\Phi_{[t-s,t-1]}^{-1}}^{2}$,
\begin{equation}
    \begin{aligned}
    &\sum_{t=s+1}^{T} \norm{x_{t}}_{\Phi_{[t-s,t-1]}^{-1}}^{2} \\
    &\leq
        2 \log(\text{det}(\Phi_{[T-s, T]})) 
        + 
        \sum_{t =s+1}^{T} 
        \norm{x_{t-s}}_{\Phi_{[t-s,t-1]}^{-1}}^{2} 
        +
        \sum_{t =1}^{T-s}
        \norm{x_{t}}_{\Phi_{[t-s,t-1]}^{-1}}^{2} \norm{x_{t-s}}_{\Phi_{[t-s,t-1]}^{-1}}^{2}
        \sin^{2}_{\Phi_{[t-s,t-1]}^{-1}}{\theta}.
    \end{aligned}
\end{equation}
Thus by combining the last two terms and extracting $\norm{x_{t-s}}_{\Phi_{[t-s,t-1]}^{-1}}^{2}$, we can get
\begin{equation}
\label{eq: C.12}
    \begin{aligned}
    \sum_{t=s+1}^{T} \norm{x_{t}}_{\Phi_{[t-s,t-1]}^{-1}}^{2} 
    \leq
        2 \log(\text{det}(\Phi_{[T-s, T]})) 
        + 
        \sum_{t =1}^{T-s} 
        \norm{x_{t-s}}_{\Phi_{[t-s,t-1]}^{-1}}^{2} 
        (\norm{x_{t}}_{\Phi_{[t-s,t-1]}^{-1}}^{2}\sin^{2}_{\Phi_{[t-s,t-1]}^{-1}}{\theta} + 1).
    \end{aligned}
\end{equation}
To make the formula simpler, we define \textit{`Forgetting Regret Term' (FRT)} term, at time window $[t-s, t-1]$ or called at time step $t$ as follows,
\begin{equation}
    \text{FRT}_{[t-s, t-1]} = \norm{x_{t-s}}_{\Phi_{[t-s,t-1]}^{-1}}^{2}(\norm{x_{t}}_{\Phi_{[t-s,t-1]}^{-1}}^{2} 
    \sin^{2}(\theta, \Phi_{[t-s,t-1]}^{-1}) + 1),
\end{equation}
where $\text{FRT}_{[t-s, t-1]} \in [0,2]$. The detailed explanation and examples of \textit{`Forgetting Regret Term' (FRT)} can be found in \ref{Appendix-Example-Rank-Swinging}.

So equation \eqref{eq: C.12} becomes
\begin{equation}
    \begin{aligned}
    \sum_{t=s+1}^{T} \norm{x_{t}}_{\Phi_{[t-s,t-1]}^{-1}}^{2} 
    \leq
        2\eta_{\text{OLS}}
        + 
        \sum_{t =s+1}^{T} 
        \text{FRT}_{[t-s, t-1]},
    \end{aligned}
\end{equation}
where $\eta_{\text{OLS}} = \log(\text{det}(\Phi_{[T-s, T]}))$ is a constant based on data time window$[T-s, T]$. By Lemma \ref{lm:C trace}, we get 
\begin{equation}
    \begin{aligned}
    \sum_{t=s+1}^{T} \norm{x_{t}}_{\Phi_{[t-s,t-1]}^{-1}}^{2} 
    \leq
        2 \left[d\log(\frac{sL^{2}}{d})
        + 
        (T-s)\right].
    \end{aligned}
\end{equation}


\end{proof}

\begin{lemma}
\label{lm:det}
\begin{equation}
    \begin{aligned}
     \text{det}(\M{I} - a a\Tra + b b\Tra) = (1 + \norm{b}^{2})(1-\norm{a}^{2}) + \langle a, b \rangle^{2}
    \end{aligned}
\end{equation}
\end{lemma}

\begin{proof} By Woodbury matrix identity
\begin{equation}
    (\M{I}-a a\Tra)^{-1} = \M{I} + \frac{a a\Tra}{1-a\Tra a}
\end{equation}
and by Matrix determinant lemma, suppose $\M{B}$ is an invertible square matrix and $u,v$ are column vectors. Then the matrix determinant lemma states that
\begin{equation}
    \text{det}(\M{B} + u v\Tra) = (1+ v\Tra \M{B}^{-1} u) \text{det}(\M{B}).
\end{equation}
So combining above two equations, we get 
\begin{equation}
    \begin{aligned}
        \text{det}(\M{I} - a a\Tra + b b\Tra) &= \left(1 + b\Tra (\M{I} - a a\Tra)^{-1} b\right)\text{det}(\M{I} - a a\Tra) \\
        & = \left(1 + b\Tra (\M{I} + \frac{a a\Tra}{1-a\Tra a})b \right) \text{det}(\M{I} - a a\Tra) \\
        & = (1 + b\Tra b + \frac{b\Tra a a\Tra b}{1-a\Tra a}) \text{det}(\M{I} - a a\Tra)\\
        & = (1 + b\Tra b + \frac{\langle a, b \rangle^{2}}{1-a\Tra a}) (1 - a\Tra a) \\
        & = (1 + \norm{b}^{2})(1-\norm{a}^{2}) + \langle a, b \rangle^{2}.
    \end{aligned}
\end{equation}
\end{proof}

\begin{lemma}
\label{lm:C trace}
\textit{(Determinant-Trace Inequality). 
Suppose $x_{1}, x_{2},..., x_{b} \in \Real^{d}$ and for any $t \in [b]$, $\norm{x_{t}}_{2} \leq L$. Let $\Phi_{[1,b]} = \sum_{t= 1}^{b} x_{t}x_{t}\Tra$. Then we have 
\begin{equation}
    \text{det}(\Phi_{[1,b]}) \leq (\frac{bL^{2}}{d})^{d}.
\end{equation}
If $\text{Rank}(\V{x}_{[1,b]}) = d$, then $\text{det}(\Phi_{[1,b]}) = det(\V{x}_{[1,b]})^2 \leq (\frac{bL^{2}}{d})^{d}$; else $\text{Rank}(\V{x}_{[1,b]}) < d$, then $\text{det}(\Phi_{[1,b]}) = 0$.}
\end{lemma}

\section{Rank Swinging Phenomenon Examples}
\label{Appendix-Example-Rank-Swinging}

\begin{itemize}[leftmargin=1.2em]
    \item Case 1: Introduce the minimum regret, then $\text{FRT} = 0$.
    \item Case 2: Introduce the maximum regret, $\text{FRT} = 2$.
    \item Case 3: Introduce medium regret scenario I, $\text{FRT} = 1$.
    \item Case 4: Introduce medium regret scenario II, $\text{FRT} = 1$.
\end{itemize}

\begin{equation}
\mathcal{M}_{1} =           
    \left(\begin{matrix}
            0 & 0 & 0 & 1 \\
            0 & 1 & 0 & 0 \\
            0 & 0 & 1 & 0 \\
            0 & 0 & 0 & 1 \\
            v_{1} & v_{2} & v_{3} & v_{4}
            \end{matrix}
    \right)
\!
\mathcal{M}_{2}= 
    \left(\begin{matrix}
          1 & 0 & 0 & 0\\
          0 & 1 & 0 & 0\\
          0 & 0 & 1 & 0\\
          0 & 0 & 0 & 1\\
          0 & 1 & 0 & 0
    \end{matrix}\right)
\!
\mathcal{M}_{3}= 
    \left(\begin{matrix}
        1 & 0 & 0 & 0\\
        0 & 1 & 0 & 0\\
        0 & 1 & 0 & 0\\
        0 & 0 & 0 & 1\\
        1 & 0 & 0 & 0
    \end{matrix}\right)
\!
\mathcal{M}_{4}= 
    \left(\begin{matrix}
        1 & 0 & 0 & 0\\
        0 & 1 & 0 & 0\\
        0 & 1 & 0 & 0\\
        0 & 0 & 0 & 1\\
        0 & 0 & 1 & 0
    \end{matrix}\right)
    \label{eq:cases example}
\end{equation}
\begin{case}(Minimum Regret)\label{case:1}
\textit{$\text{Rank}(\Phi_{[t-s,t-1]}) < d$ and the delete term $\norm{x_{t-s}}_{\Phi_{[t-s,t-1]}^{-1}}^{2} = 0$. So it won't introduce any extra regret no matter what the new data $x_{t}$ is. So $\text{FRT} = 0$.}
\end{case}

$\norm{x_{t-s}}_{\Phi_{[t-s,t-1]}^{-1}}^{2} = 0$ means that the old data $x_{t-s}$ can be fully represented by the original data memory, so delete it won't influence the representation ability of data memory.
For example, $\mathcal{M}_{1}$ in (\ref{eq:cases example}), if $s = 4, d = 4$, $\text{rank}(\V{x}_{[1,4]}) = 3$, since $x_{1}$ and $x_{4}$ are linearly correlated. By the FIFD scheme, the forgetting data is $x_{1} = (0,0,0,1)\Tra$, so $\norm{x_{1}}_{\Phi_{[1,4]}^{-1}}^{2} = 0$

\begin{case}\label{case:2}
(Maximum Regret)
\textit{ If the old data and the new data are totally dissimilar $\sin^{2}(\theta, \Phi_{[t-s,t-1]}^{-1}) = 1$ and
$\norm{x_{t}}_{\Phi_{[t-s,t-1]}^{-1}}^{2} = 1$, then the weight difference term $\norm{x_{t}}_{\Phi_{[t-s,t-1]}^{-1}}^{2}
\sin^{2}(\theta, \Phi_{[t-s,t-1]}^{-1}) + 1 = 2$. So $\text{FRT} = 2.$}
\end{case}
which means that $x_{t}$ is totally dissimilar with $x_{t-s}$ with respect to $\Phi_{[t-s,t-1]}^{-1}$, but $x_{t}$ can be represented by the rest of data memory from time window $[t-s+1, t-1]$, and the gram matrix's rank will decrease by 1. For example, $\mathcal{M}_{2}$ in (\ref{eq:cases example}), $x_{1} = (1,0,0,0)\Tra$ and $x_{5} = (0,1,0,0)\Tra$. So $\text{Rank}(\Phi_{[2,5]}) = \text{Rank}(\Phi_{[1,4]})-1$. $\norm{x_{5}}_{\Phi_{[1,4]}^{-1}}^{2}
\sin^{2}(\theta,\Phi_{[1,4]}^{-1}) = 1$.

\begin{case}\label{case:3}
(Medium Regret I)
\textit{ If the old data and the new data are perfectly similar $\sin^{2}(\theta, \Phi_{[t-s,t-1]}^{-1}) = 0$, then $\norm{x_{t}}_{\Phi_{[t-s,t-1]}^{-1}}^{2}
\sin^{2}(\theta, \Phi_{[t-s,t-1]}^{-1}) = 0$. So the weight difference term $\norm{x_{t}}_{\Phi_{[t-s,t-1]}^{-1}}^{2}\sin^{2}(\theta, \Phi_{[t-s,t-1]}^{-1}) + 1 = 1$.So $\text{FRT} = 1.$}
\end{case}
which means that $x_{t}$ is perfectly simliar with $x_{t-s}$ with respect to $\Phi_{[t-s,t-1]}^{-1}$.
For example, $\mathcal{M}_{3}$ in (\ref{eq:cases example}), $x_{1} = (1,0,0,0)\Tra$ and $x_{5} = (1,0,0,0)\Tra$, $\text{Rank}(\Phi_{[2,5]}) = \text{Rank}(\Phi_{[1,4]})$, $\sin^{2}_{\Phi_{[1,4]}^{-1}}{\theta} = 0$, $\norm{x_{5}}_{\Phi_{[1,4]}^{-1}}^{2}
\sin^{2}(\theta, \Phi_{[1,4]}^{-1}) = 0$. 
Then the weight difference term $\norm{x_{5}}_{\Phi_{[1,4]}^{-1}}^{2}
\sin^{2}(\theta, \Phi_{[1,4]}^{-1}) + 1 = 1$ no matter how large $\norm{x_{t}}_{\Phi_{[1,4]}^{-1}}^{2}$ is.

\begin{case}\label{case:4}
(Medium Regret II)
\textit{ If the new data does not lie in the space generated by $\V{x}_{[t-s, t-1]}$, which means that $\norm{x_{t}}_{\Phi_{[t-s,t-1]}^{-1}}^{2} = 0$, then $\norm{x_{t}}_{\Phi_{[t-s,t-1]}^{-1}}^{2}\sin^{2}(\theta, \Phi_{[t-s,t-1]}^{-1}) = 0$. So the weight difference term $\norm{x_{t}}_{\Phi_{[t-s,t-1]}^{-1}}^{2}\sin^{2}(\theta, \Phi_{[t-s,t-1]}^{-1}) + 1 = 1$. So $\text{FRT} = 1.$}
\end{case}

For example, $\mathcal{M}_{4}$ in (\ref{eq:cases example}), $x_{1} = (1,0,0,0)\Tra$ and $x_{5} = (0,0,1,0)\Tra$, $\text{Rank}(\Phi_{[2,5]}) = \text{Rank}(\Phi_{[1,4]})$ + 1, $\norm{x_{5}}_{\Phi_{[1,4]}^{-1}}^{2} = 0$, $\norm{x_{5}}_{\Phi_{[1,4]}^{-1}}^{2} \sin^{2}(\theta, \Phi_{[1,4]}^{-1}) = 1$. Then the weight difference term $\norm{x_{5}}_{\Phi_{[1,4]}^{-1}}^{2} \sin^{2}(\theta, \Phi_{[1,4]}^{-1}) + 1 = 1$.

\section{FIFD-Adaptive Ridge Regret}
\label{Appendix FIFD-Adaptive Ridge Regret}
\begin{theorem}
(Regret Upper Bound of The FIFD-Adaptive Ridge algorithm) \textit{The same assumption in Theorem \ref{thm:OLS Regret Bound} and if Lemma \ref{lm:Ridge Confidence Ellipsoid} holds, with probability at least $1-\delta$, , the cumulative regret satisfies:
\begin{equation}
    \begin{aligned}
        R_{T, s}(\mathcal{A}_{Ridge}) \leq 
            \sigma \kappa \nu \zeta_{\lambda}  \sqrt{(d/s) (T-s)
            [\eta_{\text{Ridge}} + (T-s) ] 
            }
    \end{aligned}
\end{equation}
where $\zeta_{\lambda} = \underset{s+1 \leq t\leq T}{\mathrm{max}} \frac{\norm{\V{x}_{[t-1,t-s]}}_{\infty}}{\phi^2_{\lambda, [t-s,t-1]}}$ is the maximum adaptive constant over time, $\eta_{\text{Ridge}} = d \log (sL^{2}/d + \lambda_{[T-s, T-1]}) - \log{C_{2}(\V{\phi})}$ is a constant related to the last data memory, $C_{2}(\V{\phi}) = \prod_{t=s+1}^{T}(1 + \frac{s}{\phi^2_{\lambda, [t-s+1,t]} + \lambda_{\Delta, [t-s+1,t]}} \lambda_{\Delta, [t-s+1,t]})$ is a constant close to 1, and $\lambda_{\Delta, [t-s+1,t]} = \lambda_{[t-s+1,t]} - \lambda_{[t-s,t-1]}$ represents the fluctuation of $\lambda$ over time steps.}
\end{theorem}

\begin{proof}
Here we use $l_{\lambda, t}$ denote the instantaneous \textit{absolute loss} at time $t$ using \texttt{FIFD-Adaptive Redge} algorithm.
Let's decompose the instantaneous absolute loss as follows:
\begin{equation}
    \begin{aligned}
        l_{\lambda,t} &= |\langle \hat{\theta}_{\lambda, [t-s,t-1]}, x_{t} \rangle - \langle \theta_{\star}, x_{t} \rangle | \\
        & = |\langle \hat{\theta}_{\lambda, [t-s,t-1]} - \theta_{\star}, x_{t} \rangle| \\
        & = |[\hat{\theta}_{\lambda, [t-s,t-1]} - \theta_{\star}]\Tra x_{t}| \\
        & = |[\hat{\theta}_{\lambda, [t-s,t-1]} - \theta_{\star}]\Tra \Phi_{\lambda, [t-s,t-1]}^{\frac{1}{2}} \Phi_{\lambda, [t-s,t-1]}^{-\frac{1}{2}} x_{t}| \\
        & = |[\hat{\theta}_{\lambda, [t-s,t-1]} - \theta_{\star}]\Tra \Phi_{\lambda, [t-s,t-1]}^{\frac{1}{2}}| \times |\Phi_{\lambda, [t-s,t-1]}^{-\frac{1}{2}} x_{t}| \\
        & \leq \norm{\hat{\theta}_{\lambda, [t-s,t-1]} - \theta_{\star}}_{\Phi_{\lambda, [t-s,t-1]}} \norm{x_{t}}_{\Phi_{\lambda, [t-s,t-1]}^{-1}} \\
        & \leq \sqrt{\beta_{\lambda, [t-s,t-1]}(\delta)}  \norm{x_{t}}_{\Phi_{\lambda, [t-s,t-1]}^{-1}}
    \end{aligned}
\end{equation}
where the last step from Lemma \ref{lm:Appendix Ridge Confidence Ellipsoid}. To compute the regret, we first denote $r_{\lambda, t}$ as the instantaneous regret at time $t$. Let's decompose the instantaneous regret as follows,
$r_{\lambda, t} = \langle \hat{\theta}_{\lambda, [t-s,t-1]}, x_{t} \rangle - \langle \theta_{\star}, x_{t} \rangle \leq \sqrt{\beta_{\lambda, [t-s,t-1]}(\delta)}  \norm{x_{t}}_{\Phi_{\lambda, [t-s,t-1]}^{-1}}$,
where the inequality is from Lemma \ref{lm:Appendix Ridge Confidence Ellipsoid}. 

Thus, with probability at least $1-\delta$, for all $T>s$,
\begin{equation}
    \begin{aligned}
        R_{T, s}(\mathcal{A}) &= \sqrt{(T-s)\sum_{t = s+1}^{T} r_{\lambda, t}^2} \\
        & \leq 
            \sqrt{(T-s) \sum_{t = s+1}^{T}  \beta_{\lambda, [t-s,t-1]}(\delta)\norm{x_{t}}_{\Phi_{\lambda, [t-s,t-1]}^{-1}}^{2}} \\
         & \leq 
            \sqrt{2(T-s)
            \underset{s+1 \leq t\leq T}{\mathrm{max}} \beta_{\lambda, [t-s,t-1]}(\delta) 
            \left( 2 \left[d \log (\frac{sL^{2}}{d} + \lambda_{[T-s, T-1]}) - \log{C_{2}(\V{\phi})} + (T-s)\right] 
            \right) 
            }
    \end{aligned}
\end{equation}
where the last step we use Lemma \ref{lm:Appendix Ridge Deletion Regret} to deal with the online forgetting process to get the summation of term $\norm{x_{t}}^{2}_{\Phi_{\lambda, [t-s,t-1]}^{-1}}$. for each time step. So we have the cumulative regret upper bound for the algorithm \texttt{FIFD-Adaptive Ridge} as follows,
\begin{equation}
    \begin{aligned}
      R_{T, s}(\mathcal{A}_{\text{Ridge}})  & =\sigma \kappa \nu \zeta_{\lambda}  \sqrt{(d/s) (T-s)
            [\eta_{\text{Ridge}} + (T-s) ]}
    \end{aligned}
\end{equation}
where we use the Lemma \ref{lm:Appendix Ridge Confidence Ellipsoid} to get the maximum ellipsoid confidence $\underset{t\in[s+1,T]}{\mathrm{max}} \beta_{\lambda, [t-s,t-1]}(\delta)$.
\end{proof}

\begin{lemma}
\label{lm:Appendix Ridge Deletion Regret}
\textit{The cumulative regret of FIFD-Adaptive Ridge is partially determined by FRT-Ridge at each time step,
\begin{equation}
    \begin{aligned}
    \sum_{t=s+1}^{T} \norm{x_{t}}_{\Phi_{\lambda, [t-s,t-1]}^{-1}}^{2} 
    \leq
        2\eta_{\text{Ridge}}
        + 
        \sum_{t =s+1}^{T} 
        \text{FRT}_{\lambda, [t-s, t-1]}
    \leq  2 \left[d \log (\frac{sL^{2}}{d} + \lambda_{[T-s, T-1]}) - \log{C_{2}(\V{\phi})}
        + 
        (T-s)\right] 
    \end{aligned}
\end{equation}
where $\eta_{\text{Ridge}}= d \log (sL^{2}/d + \lambda_{[T-s, T-1]}) - \log{C_{2}(\V{\phi})}$ is a constant based on the limited data time window$[T-s, T]$.}
\end{lemma}

\begin{proof}
Here we still use $a = t-s, b = t-1$ for the sake of simplicity. Elementary algebra gives
\begin{equation}
    \begin{aligned}
        \text{det}&(\Phi_{\lambda, [a+1,b+1]})  = \text{det}(\Phi_{\lambda, [a,b]} + x_{b+1}x_{b+1}\Tra - x_{a}x_{a}\Tra + \lambda_{\Delta,[a+1,b+1]}), 
    \end{aligned}
\end{equation}
where $\lambda_{\Delta,[a+1,b+1]} = \lambda_{[a+1,b+1]} - \lambda_{[a,b]}$ and the choice of $\lambda_{[a+1,b+1]}$ $\lambda_{[a,b]}$ is given by Lemma \ref{eq: Ridge hyperparameter selection}. Then $\text{det}(\Phi_{\lambda, [a+1,b+1]})$ equals to 
\begin{equation}
\label{eq: Appendix E.7}
\begin{aligned}
        & = \text{det}(\Phi_{\lambda, [a,b]}^{1/2}
        (\M{I} + \Phi_{\lambda, [a,b]}^{-1/2}(x_{b+1}x_{b+1}\Tra - x_{a}x_{a}\Tra + \lambda_{\Delta,[a+1,b+1]}\Phi_{\lambda, [a,b]}^{-1} )\Phi_{\lambda, [a,b]}^{-1/2})
        \Phi_{\lambda, [a,b]}^{1/2}) \\
        & = \text{det}(\Phi_{\lambda, [a,b]})\text{det}(\M{I} + \Phi_{\lambda, [a,b]}^{-1/2}(x_{b+1}x_{b+1}\Tra - x_{a}x_{a}\Tra)\Phi_{\lambda, [a,b]}^{-1/2}+ \lambda_{\Delta,[a+1,b+1]}\Phi_{\lambda, [a,b]}^{-1}) \\
        & = \text{det}(\Phi_{\lambda, [a,b]})\text{det}(\M{I} + \Phi_{\lambda, [a,b]}^{-1/2}x_{b+1}x_{b+1}\Tra \Phi_{\lambda, [a,b]}^{-1/2} - 
        \Phi_{\lambda, [a,b]}^{-1/2}x_{a}x_{a}\Tra \Phi_{\lambda, [a,b]}^{-1/2}
        + \lambda_{\Delta,[a+1,b+1]}\Phi_{\lambda, [a,b]}^{-1}) \\
        & = \text{det}(\Phi_{\lambda, [a,b]}) \text{det}\left(\M{I} + (\Phi_{\lambda, [a,b]}^{-1/2} x_{b+1})(\Phi_{\lambda, [a,b]}^{-1/2} x_{b+1})\Tra - 
        (\Phi_{\lambda, [a,b]}^{-1/2} x_{a})(\Phi_{\lambda, [a,b]}^{-1/2} x_{a})\Tra + \lambda_{\Delta,[a+1,b+1]}\Phi_{\lambda, [a,b]}^{-1}
        \right). 
\end{aligned}
\end{equation}
We first compute the second determinant of equation \eqref{eq: Appendix E.7} and denote $\M{B} = \M{I} + (\Phi_{\lambda, [a,b]}^{-1/2} x_{b+1})(\Phi_{\lambda, [a,b]}^{-1/2} x_{b+1})\Tra - 
(\Phi_{\lambda, [a,b]}^{-1/2} x_{a})(\Phi_{\lambda, [a,b]}^{-1/2} x_{a})\Tra$ and use the matrix technique $\text{det}(\M{A} + \epsilon \M{X}) \approx \text{det}(\M{A}) + \text{det}(\M{A}) \text{tr}(\M{A}^{-1} \M{X})\epsilon + \mathcal{O}(\epsilon^{2})$. So we have 
\begin{equation}
    \begin{aligned}
        \text{det}&(\M{B}  + \lambda_{\Delta,[a+1,b+1]}\Phi_{\lambda, [a,b]}^{-1}) 
        \approx \text{det}(\M{B}) + \text{det}(\M{B})\text{tr}(A^{-1}\Phi_{\lambda, [a,b]}^{-1}) \lambda_{\Delta,[a+1,b+1]} +\mathcal{O}(\lambda^{2}_{\Delta,[a+1,b+1]})
    \end{aligned}
\end{equation}
So the second determinant part of equation becomes  
\begin{equation}
    \begin{aligned}     
    & \text{det}(\M{B}  + \lambda_{\Delta,[a+1,b+1]}\Phi_{\lambda, [a,b]}^{-1}) \\
        & \approx \text{det}(\M{B})\left[1 + \text{tr}(\M{B}^{-1}\Phi_{\lambda, [a,b]}^{-1}) \lambda_{\Delta,[a+1,b+1]}\right] \\
        & = \text{det}(\M{B})\left[1 + \text{tr}[(\Phi_{\lambda, [a,b]}\M{B})^{-1}] \lambda_{\Delta,[a+1,b+1]}\right] \\
        & = \text{det}(\M{B})\left[1 + \text{tr}[(\Phi_{\lambda, [a,b]} + x_{b+1}x_{b+1}\Tra - x_{a}x_{a}\Tra)^{-1}] \lambda_{\Delta,[a+1,b+1]}\right] \\
        & = \text{det}(\M{B})\left[1 + \text{tr}[(\Phi_{\lambda, [a+1,b+1]} - \lambda_{\Delta,[a+1,b+1]}\M{I})^{-1}] \lambda_{\Delta,[a+1,b+1]}\right] \\
        & \geq \text{det}(\M{B})\left[1 + 
        \frac{d}{\phi^{2}_{\lambda, [a+1,b+1]} - \lambda_{\Delta,[a+1,b+1]}}\lambda_{\Delta,[a+1,b+1]}
        \right]
    \end{aligned}
\end{equation}
where the last step we use the minimum eigenvalue of $\Phi_{\lambda, [a+1,b+1]}$ to get the inequality and we denote $c_{2,[a+1,b+1]} = 1 +
\frac{d}{\phi^{2}_{\lambda, [a+1,b+1]} - \lambda_{\Delta,[a+1,b+1]}}\lambda_{\Delta,[a+1,b+1]}$. So when we go back to equation \eqref{eq: Appendix E.7},  $\text{det}(\Phi_{\lambda, [a+1,b+1]})$ equals to 
\begin{equation}
    \begin{aligned}
        & = \text{det}(\Phi_{\lambda, [a,b]})\text{det}\left(\M{I} + (\Phi_{\lambda, [a,b]}^{-1/2} x_{b+1})(\Phi_{\lambda, [a,b]}^{-1/2} x_{b+1})\Tra - 
        (\Phi_{\lambda, [a,b]}^{-1/2} x_{a})(\Phi_{\lambda, [a,b]}^{-1/2} x_{a})\Tra
        \right) c_{2,[a+1,b+1]}\\
        & = \text{det}(\Phi_{\lambda, [a,b]}) \left[(1 + \norm{\Phi_{\lambda, [a,b]}^{-1/2} x_{b+1}}^{2})(1 - \norm{\Phi_{\lambda, [a,b]}^{-1/2} x_{a}}^{2}) + \langle \Phi_{\lambda, [a,b]}^{-1/2} x_{b+1}, \Phi_{\lambda, [a,b]}^{-1/2} x_{a} \rangle^{2}\right]c_{2,[a+1,b+1]}.
    \end{aligned}
\end{equation}
where the last step use lemma \ref{lm:det}, same as the technique using in obtaining the regret upper bound of \texttt{FIFD-OLS}. So we have 
\begin{equation}
    \label{eq: det cos E.7}
    \begin{aligned}
        & = \text{det}(\Phi_{\lambda, [a,b]}) 
        \left( 
        (1 + \norm{x_{b+1}}_{\Phi_{\lambda, [a,b]}^{-1}}^{2})
        (1 - \norm{x_{a}}_{\Phi_{\lambda, [a,b]}^{-1}}^{2})
        + \langle x_{b+1}, x_{a} \rangle_{\Phi_{\lambda, [a,b]}^{-1}}^{2}
        \right) c_{2,[a+1,b+1]}\\
        & = \text{det}(\Phi_{\lambda, [a,b]})  
        \left( 1 + \norm{x_{b+1}}_{\Phi_{\lambda, [a,b]}^{-1}}^{2} - \norm{x_{a}}_{\Phi_{\lambda, [a,b]}^{-1}}^{2}  +
        \langle x_{b+1}, x_{a} \rangle_{\Phi_{\lambda, [a,b]}^{-1}}^{2}
        - \norm{x_{b+1}}_{\Phi_{\lambda, [a,b]}^{-1}}^{2} \norm{x_{a}}_{\Phi_{\lambda, [a,b]}^{-1}}^{2} \right) c_{2,[a+1,b+1]}\\
        & = \text{det}(\Phi_{\lambda, [a,b]})  
        \left(
        1 + \norm{x_{b+1}}_{\Phi_{\lambda, [a,b]}^{-1}}^{2} - \norm{x_{a}}_{\Phi_{\lambda, [a,b]}^{-1}}^{2} + 
        \norm{x_{b+1}}_{\Phi_{\lambda, [a,b]}^{-1}}^{2} \norm{x_{a}}_{\Phi_{\lambda, [a,b]}^{-1}}^{2}
        (\cos^{2}_{\Phi_{\lambda, [a,b]}^{-1}}{\theta} - 1) 
        \right) 
        c_{2,[a+1,b+1]}.
    \end{aligned}
\end{equation}

where $\cos_{\Phi_{\lambda, [a,b]}^{-1}}{\theta} = \frac{\langle x_{b+1}, x_{a} \rangle_{\Phi_{\lambda, [a,b]}^{-1}}}{\norm{x_{b+1}}_{\Phi_{\lambda, [a,b]}^{-1}} \norm{x_{a}}_{\Phi_{\lambda, [a,b]}^{-1}}}$, which measures the similarity of the vector of $x_{b+1}$ and $x_{a}$ with respect to $\Phi_{\lambda, [a,b]}^{-1}$. If $\cos^{2}_{\Phi_{\lambda, [a,b]}^{-1}}{\theta} = 1$, that means the incoming data $x_{b+1}$ and the deleted data $x_{a}$ are same. If $\cos^{2}_{\Phi_{\lambda, [a,b]}^{-1}}{\theta} = 0$, that means the incoming data and the deleted data are totally different with respect to $\Phi_{\lambda, [a,b]}^{-1}$. 

Now we switch back to the notation time step $t$ and $t-s = a, t-1 =b$. At time step $T$, and we use equation \eqref{eq: det cos c.7}, we have 
\begin{equation}
\label{eq: Appendix E.12}
    \begin{aligned}
        \text{det}(\Phi_{\lambda, [T-s, T]}) 
         = \prod_{t=s+1}^{T} 
        (&
        1 + \norm{x_{t}}_{\Phi_{\lambda, [t-s,t-1]}^{-1}}^{2} 
        - \norm{x_{t-s}}_{\Phi_{\lambda, [t-s,t-1]}^{-1}}^{2} \\ 
        & + \norm{x_{t}}_{\Phi_{\lambda, [t-s,t-1]}^{-1}}^{2} \norm{x_{t-s}}_{\Phi_{\lambda, [t-s,t-1]}^{-1}}^{2}
        (\cos^{2}_{\Phi_{\lambda, [t-s,t-1]}^{-1}}{\theta} - 1)
        )c_{2,[t-s+1,t]}.
    \end{aligned}
\end{equation}

Taking $\log$ to both sides of equation \eqref{eq: Appendix E.12},  we can get
\begin{equation}
\label{eq: Appendix E.13}
    \begin{aligned}
        &\log(\text{det}(\Phi_{\lambda, [T-s, T]})) \\
        & = \log(C_{2}(\phi)) + \sum_{t=s+1}^{T}
        \log 
        (
        1 + \norm{x_{t}}_{\Phi_{\lambda, [t-s,t-1]}^{-1}}^{2} 
        - \norm{x_{t-s}}_{\Phi_{\lambda, [t-s,t-1]}^{-1}}^{2} \\ 
        &\hspace{4cm} + \norm{x_{t}}_{\Phi_{\lambda, [t-s,t-1]}^{-1}}^{2} \norm{x_{t-s}}_{\Phi_{\lambda, [t-s,t-1]}^{-1}}^{2}
        (\cos^{2}_{\Phi_{\lambda, [t-s,t-1]}^{-1}}{\theta} - 1)
        ),
    \end{aligned}
\end{equation}
where $C_{2}(\phi) = \prod_{t=s+1}^{T}c_{2,[t-s+1,t]}$. Combining the inequality of $\log(1+x) > \frac{x}{1+x}$ which holds when $x > -1$,
we first consider each part of the product and use the dissimilarity measure $\sin^{2}_{\Phi_{\lambda, [a,b]}^{-1}}{\theta} =  1- \cos^{2}_{\Phi_{\lambda, [a,b]}^{-1}}{\theta}$. So we get 
\begin{equation}
\label{eq: Appendix E.14}
    \begin{aligned}
        \log &
        \left(
        1 + \norm{x_{t}}_{\Phi_{\lambda, [t-s,t-1]}^{-1}}^{2} - \norm{x_{t-s}}_{\Phi_{\lambda, [t-s,t-1]}^{-1}}^{2} - \norm{x_{t}}_{\Phi_{\lambda, [t-s,t-1]}^{-1}}^{2} \norm{x_{t-s}}_{\Phi_{\lambda, [t-s,t-1]}^{-1}}^{2}
        \sin^{2}_{\Phi_{\lambda, [t-s,t-1]}^{-1}}{\theta}
        \right) \\
        &> 
        \frac{\norm{x_{t}}_{\Phi_{\lambda, [t-s,t-1]}^{-1}}^{2} - \norm{x_{t-s}}_{\Phi_{\lambda, [t-s,t-1]}^{-1}}^{2} - \norm{x_{t}}_{\Phi_{\lambda, [t-s,t-1]}^{-1}}^{2} \norm{x_{t-s}}_{\Phi_{\lambda, [t-s,t-1]}^{-1}}^{2}
        \sin^{2}_{\Phi_{\lambda, [t-s,t-1]}^{-1}}{\theta}
        }
        {1 
        + \norm{x_{t}}_{\Phi_{\lambda, [t-s,t-1]}^{-1}}^{2} 
        - \norm{x_{t-s}}_{\Phi_{\lambda, [t-s,t-1]}^{-1}}^{2}
        - \norm{x_{t}}_{\Phi_{\lambda, [t-s,t-1]}^{-1}}^{2} \norm{x_{t-s}}_{\Phi_{\lambda, [t-s,t-1]}^{-1}}^{2}
        \sin^{2}_{\Phi_{\lambda, [t-s,t-1]}^{-1}}{\theta}
        } \\
        &> \frac{1}{2} 
        \left[\norm{x_{t}}_{\Phi_{\lambda, [t-s,t-1]}^{-1}}^{2} - \norm{x_{t-s}}_{\Phi_{\lambda, [t-s,t-1]}^{-1}}^{2} - \norm{x_{t}}_{\Phi_{\lambda, [t-s,t-1]}^{-1}}^{2} \norm{x_{t-s}}_{\Phi_{\lambda, [t-s,t-1]}^{-1}}^{2}
        \sin^{2}_{\Phi_{\lambda, [t-s,t-1]}^{-1}}{\theta}
        \right].
    \end{aligned}
\end{equation}
Therefore, we can provide a upper bound of the term $\sum_{t=s+1}^{T} \norm{x_{t}}_{\Phi_{\lambda, [t-s,t-1]}^{-1}}^{2}$ combing equation \eqref{eq: Appendix E.12}\eqref{eq: Appendix E.13}\eqref{eq: Appendix E.14},
\begin{equation}
    \begin{aligned}
    \sum_{t=s+1}^{T} &\norm{x_{t}}_{\Phi_{\lambda, [t-s,t-1]}^{-1}}^{2} 
    \leq
    2 \left[\log(\text{det}(\Phi_{\lambda, [T-s, T]})) - \log(C_{2}(\phi)) \right] 
        + 
        \sum_{t =s+1}^{T} 
        \norm{x_{t-s}}_{\Phi_{\lambda, [t-s,t-1]}^{-1}}^{2} \\ 
        & \hspace{4cm}+
        \sum_{t =1}^{T-s}
        \norm{x_{t}}_{\Phi_{\lambda, [t-s,t-1]}^{-1}}^{2} \norm{x_{t-s}}_{\Phi_{\lambda, [t-s,t-1]}^{-1}}^{2}
        \sin^{2}_{\Phi_{\lambda, [t-s,t-1]}^{-1}}{\theta}.
    \end{aligned}
\end{equation}
Thus by combining the last two terms and extracting $\norm{x_{t-s}}_{\Phi_{\lambda, [t-s,t-1]}^{-1}}^{2}$, we can get
\begin{equation}
\label{eq: C.12}
    \begin{aligned}
    &\sum_{t=s+1}^{T} \norm{x_{t}}_{\Phi_{\lambda, [t-s,t-1]}^{-1}}^{2}\\ 
    \leq
    &2 \left[\log(\text{det}(\Phi_{\lambda, [T-s, T]})) - \log(C_{2}(\phi)) \right] 
        + 
        \sum_{t =1}^{T-s} 
        \norm{x_{t-s}}_{\Phi_{\lambda, [t-s,t-1]}^{-1}}^{2} 
        (\norm{x_{t}}_{\Phi_{\lambda, [t-s,t-1]}^{-1}}^{2}\sin^{2}_{\Phi_{\lambda, [t-s,t-1]}^{-1}}{\theta} + 1).
    \end{aligned}
\end{equation}
To make the formula simpler, we define \textit{`Forgetting Regret Term' (FRT)} term with respect to adaptive ridge parameter $\lambda$, at time window $[t-s, t-1]$ or called at time step $t$ as follows,
\begin{equation}
    \text{FRT}_{\lambda, [t-s, t-1]} = \norm{x_{t-s}}_{\Phi_{\lambda, [t-s,t-1]}^{-1}}^{2}(\norm{x_{t}}_{\Phi_{\lambda, [t-s,t-1]}^{-1}}^{2} 
    \sin^{2}(\theta, \Phi_{\lambda, [t-s,t-1]}^{-1}) + 1),
\end{equation}
where $\text{FRT}_{\lambda, [t-s, t-1]} \in [0,2]$. The detailed explanation and examples of \textit{`Forgetting Regret Term' (FRT)} can be found in \ref{Appendix-Example-Rank-Swinging}.

So equation \eqref{eq: C.12} becomes
\begin{equation}
    \begin{aligned}
    \sum_{t=s+1}^{T} \norm{x_{t}}_{\Phi_{\lambda, [t-s,t-1]}^{-1}}^{2} 
    \leq
        2\eta_{\text{Ridge}}
        + 
        \sum_{t =s+1}^{T} 
        \text{FRT}_{\lambda, [t-s, t-1]},
    \end{aligned}
\end{equation}
where $\eta_{\text{Ridge}} = \log(\text{det}(\Phi_{\lambda, [T-s, T]})) - \log(C_{2}(\phi))$ is a constant based on data time window$[T-s, T]$. By Lemma \ref{lm:C trace}
, we get 
\begin{equation}
    \begin{aligned}
    \sum_{t=s+1}^{T} \norm{x_{t}}_{\Phi_{\lambda, [t-s,t-1]}^{-1}}^{2} 
    \leq
        2 \left[d \log (\frac{sL^{2}}{d} + \lambda_{[T-s, T-1]}) - \log{C_{2}(\V{\phi})}
        + 
        (T-s)\right].
    \end{aligned}
\end{equation}

\end{proof}

\newpage

\section{Online Incremental Update for FIFD-OLS}
\label{Appendix Online Update}
 The \texttt{FIFD-OLS} $\hat{\theta}_{[a,b]}$ estimator based on data from decision point $a$ to $b$ is defined as
\begin{equation}\label{eq:LSE_formula}
    \hat{\theta}_{[a,b]} = \Phi_{[a,b]}^{-1}\big[ \sum_{i=a}^{b}y_{i}x_{i}\big].
\end{equation}
We present an incremental update formula from $\hat{\theta}_{[a,b]}$ to
$\hat{\theta}_{[a+1,b+1]}$:
\begin{theorem}[Incremental Update for length $s$ \texttt{FIFD}-least square estimator]
\begin{equation}
    \hat{\theta}_{[a+1,b+1]}
    =
    f(\Phi_{[a,b]}^{-1})
    \cdot 
    g(
    \hat{\theta}_{[a,b]},
    \Phi_{[a,b]}),
\end{equation}
where $f(A)$ is defined as
\begin{equation}
    f(A) = \Gamma(A)
    -(x_{a}^\top \Gamma-1)^{-1}
    \big[ 
    \Gamma(A)
    x_{a}
    x_{a}^\top 
    \Gamma(A)
    \big] 
\end{equation}
with $\Gamma(A)\equiv
A-(x_{b+1}^\top A x_{b+1}+1)^{-1}
    \big[ 
    A
    x_{b+1}
    x_{b+1}^\top 
    A
    \big]
$ and $g(\theta)$ is defined as
\begin{equation}
    g(\theta, \Phi) 
    = \Phi\theta
    +y_{b+1}x_{b+1}
    -y_{a}x_{a}.
\end{equation}

\end{theorem}
\begin{proof}
We break the proof into 2 steps. The first step is to update the inverse of sample covariance matrix from $\Phi_{[a,b]}^{-1}$
to $\Phi_{[a+1,b+1]}^{-1}$. The second step is simple algebra and the definition of least square estimator \eqref{eq:LSE_formula}.

\textbf{Step 01.}
Bases on Lemmas \ref{lm:updata_step1} and \ref{lm:updata_step2}, we can do
incremental update on the inverse of sample covariance matrix from
 $\Phi_{[a,b]}^{-1}$ 
to $\Phi_{[a+1,b+1]}^{-1}$ as
\begin{equation}\label{eq:lse_update_scheme}
\begin{cases}
    \Phi_{[a+1,b+1]}^{-1}
    &=
    \Gamma(\Phi_{[a,b]}^{-1})
    -(x_{a}^\top \Gamma(\Phi_{[a,b]}^{-1})-1)^{-1}
    \big[ 
    \Gamma(\Phi_{[a,b]}^{-1})
    x_{a}
    x_{a}^\top 
    \Gamma(\Phi_{[a,b]}^{-1})
    \big] \\
    \Gamma(\Phi_{[a,b]}^{-1}) &= \Phi_{[a,b]}^{-1}
    -(x_{b+1}^\top \Phi_{[a,b]}^{-1} x_{b+1}+1)^{-1}
    \big[ 
    \Phi_{[a,b]}^{-1}
    x_{b+1}
    x_{b+1}^\top 
    \Phi_{[a,b]}^{-1}
    \big]
\end{cases}.
\end{equation}
We write $\Phi_{[a+1,b+1]}^{-1} = f(\Phi_{[a,b]}^{-1}),$ where the function $f(\cdot)$ is defined by the update scheme \eqref{eq:lse_update_scheme}.

\textbf{Step 02.}
Put into 
\eqref{eq:LSE_formula} that
$
\sum_{i=a+1}^{b+1}y_{i} x_{i}
=
\Phi_{[a,b]}\hat{\theta}_{[a,b]}
+y_{b+1}x_{b+1}
    -y_{a}x_{a}.
$
\end{proof}

\begin{lemma}[Matrix Inversion Formula]
\label{lm:matrix_inv_formula}
\begin{equation}
[A+BCD]^{-1}
=
A^{-1}
-
A^{-1}
B
[DA^{-1}B+C^{-1}]^{-1}
D
A^{-1}.
\end{equation}
\end{lemma}

\begin{lemma}[Update Scheme-Step 01]
\label{lm:updata_step1}
\begin{equation}
    \Phi_{[a,b+1]}^{-1}
    =
    \Phi_{[a,b]}^{-1}
    -(x_{b+1}^\top \Phi_{[a,b]}^{-1} x_{b+1}+1)^{-1}
    \big[ 
    \Phi_{[a,b]}^{-1}
    x_{b+1}
    x_{b+1}^\top 
    \Phi_{[a,b]}^{-1}
    \big].
\end{equation}
\end{lemma}
\begin{proof}
Note $\Phi_{[a,b+1]}=\Phi_{[a,b]}+x_{b+1}x_{b+1}^\top$.
Take $A=\Phi_{[a,b]}, B=x_{b+1}, C = 1, D= x_{b+1}^\top$ in Lemma \ref{lm:matrix_inv_formula}.
\end{proof}

\begin{lemma}[Update Scheme-Step 02]\label{lm:updata_step2}
\begin{equation}
    \Phi_{[a+1,b+1]}^{-1}
    =
    \Phi_{[a,b+1]}^{-1}
    -(x_{a}^\top \Phi_{[a,b+1]}^{-1} x_{a}-1)^{-1}
    \big[ 
    \Phi_{[a,b+1]}^{-1}
    x_{a}
    x_{a}^\top 
    \Phi_{[a,b+1]}^{-1}
    \big].
\end{equation}
\end{lemma}
\begin{proof}
Note $\Phi_{[a+1,b+1]}=\Phi_{[a,b+1]}-x_{a}x_{a}^\top$.
Take $A=\Phi_{[a,b+1]}, B=x_{a}, C = -1, D= x_{a}^\top$ in Lemma \ref{lm:matrix_inv_formula}.
\end{proof}

\section{Additional Simulation Results}
\label{Appendix Additional Simulation Results}

In Figure \ref{fig:Ridge L2}, we show the $L_2$ error of the \texttt{FIFD-Adaptive Ridge} method and the Fixed ridge method. From all of the twelve subplots, we can see that $L_2$ error of the adaptive ridge method can be bounded by 1. As we can see, as the noise level $\sigma$ increases, the $L_2$ error increases. If we increase the constant memory limit $s$, the $L_2$ error will decrease. Besides, we can see the error bar of the \texttt{FIFD-Adaptive Ridge} is relative narrower than the Fixed ridge over all settings. Without any prior knowledge, we can achieve the best or close to the best result compared with the Fixed ridge method with prior knowledge of $\lambda$.

In Figure \ref{fig:Ridge lambda}, we show the choice of hyperparameter $\lambda$ over different time steps. Since the hyperparameter $\lambda$ is calculated over each time interval $[t-s, t-1], \forall t\in [s+1, T]$. Thus, it is adaptive to the incoming data compared with Fixed ridge method with pre-defined hyperparameter $\lambda$.

\begin{figure*}[t!]
\centering
    \includegraphics[scale=0.7]{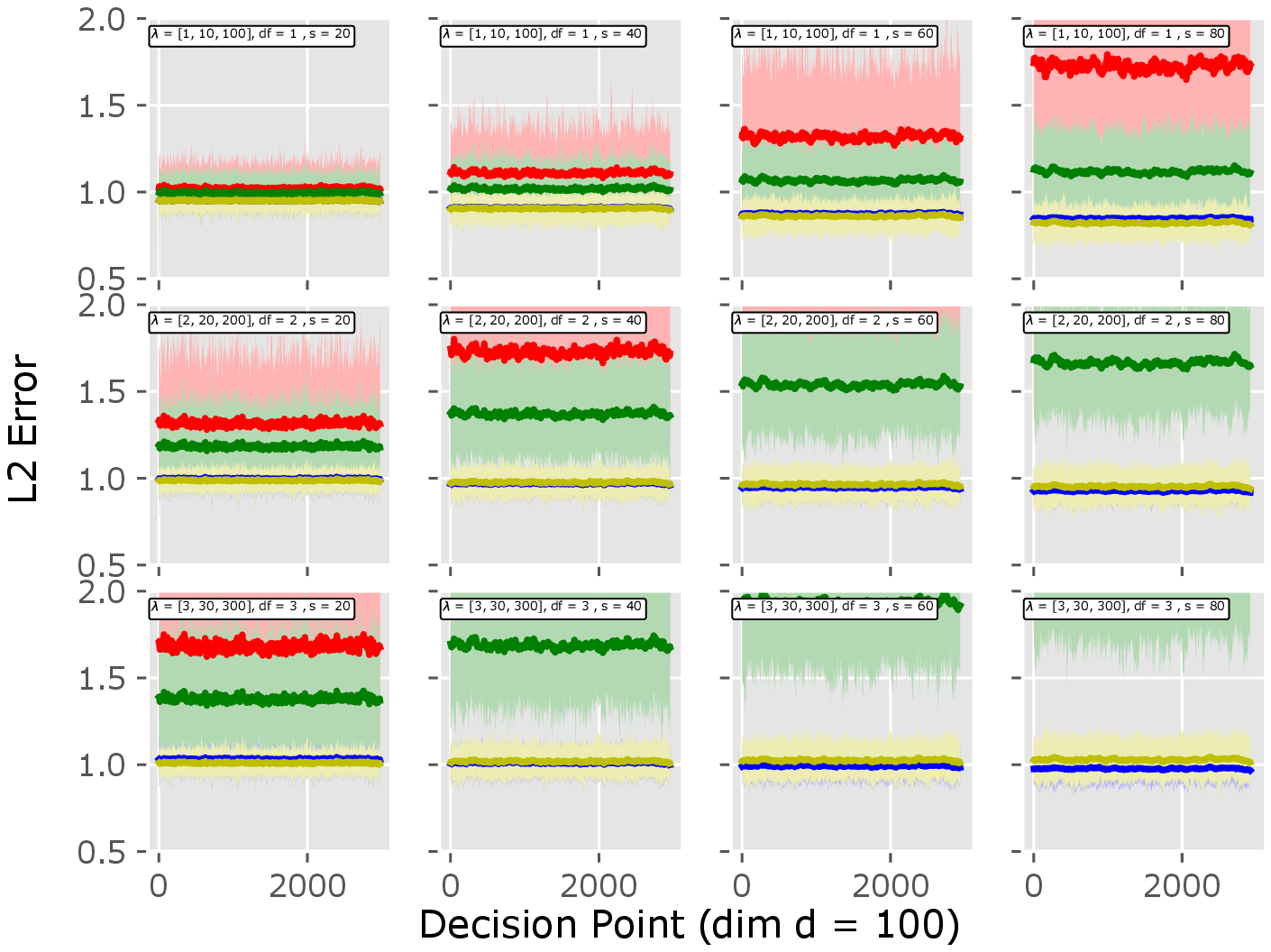}\vspace{-0.7em}
    \caption{Comparison of $L_2$ error between \texttt{FIFD-Adaptive Ridge} method and Fixed Ridge method. The error bars represent the standard error of the mean regret over 100 runs.}
    \vspace{-0.7em}
\label{fig:Ridge L2}
\end{figure*} 

\begin{figure*}[t!]
\centering
    \includegraphics[scale=0.7]{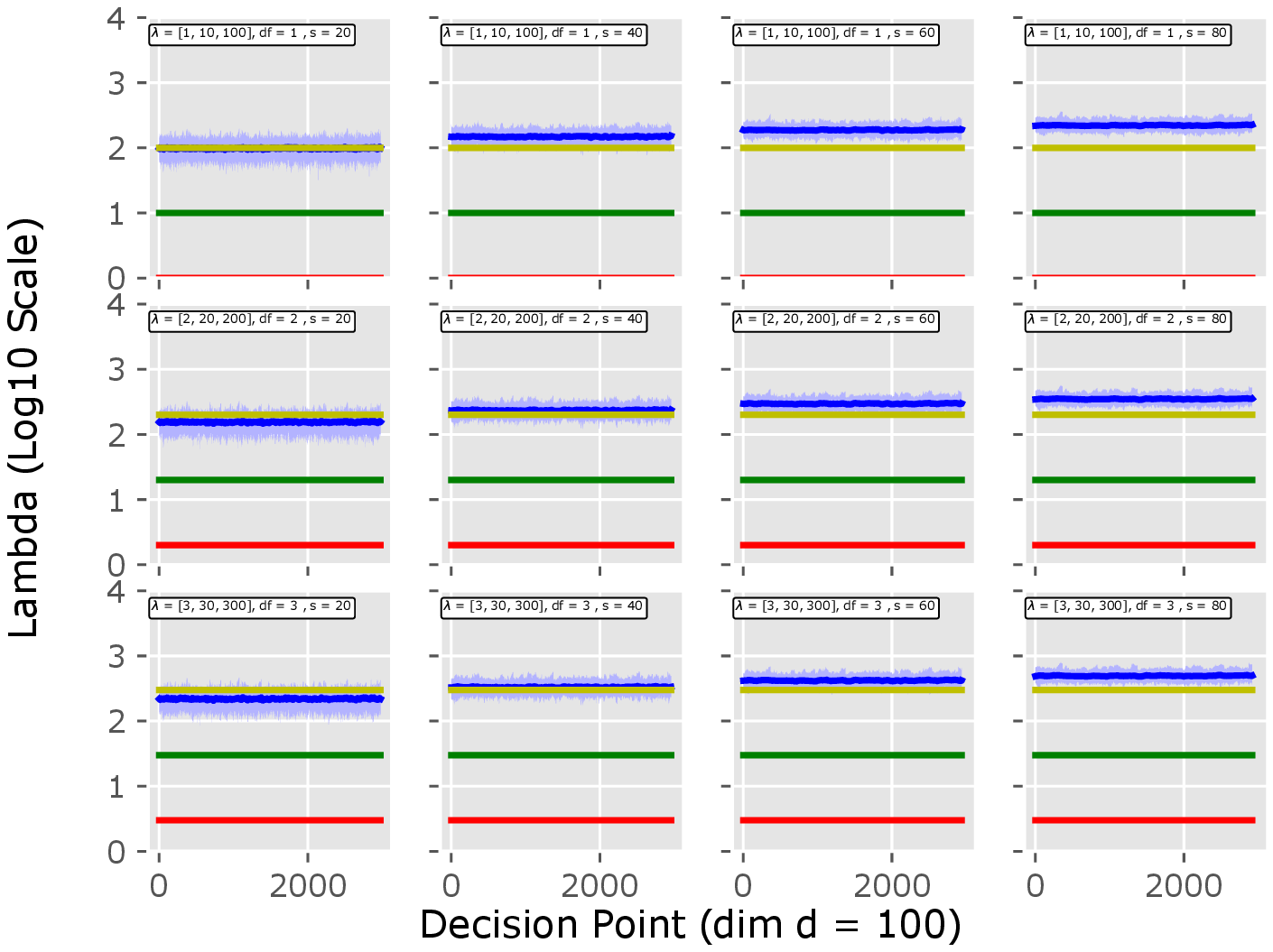}\vspace{-0.7em}
    \caption{The choice of \texttt{FIFD-Adaptive Ridge} and Fixed Ridge $\lambda$.}
    \vspace{-0.7em}
\label{fig:Ridge lambda}
\end{figure*}

\end{document}